\documentclass[11pt]{article}

\usepackage[T1]{fontenc}
\usepackage{lmodern}
\usepackage[utf8]{inputenc}
\usepackage[a4paper,margin=25mm]{geometry}
\usepackage{amsmath,amssymb,amsfonts,amsthm,mathtools}
\usepackage{bm}
\usepackage{booktabs}
\usepackage{graphicx}
\usepackage[hidelinks]{hyperref}
\hypersetup{
  pdftitle={Cone Rayleigh Levels: Finite Perturbations and Certified Control},
  pdfauthor={Yavdat S. Il'yasov and Nur F. Valeev}
}
\usepackage{cite}
\newcommand{\R}{\mathbb{R}}

\newcommand{\lambdaup}{\overline{\lambda}}
\newcommand{\lambdalow}{\underline{\lambda}}

\DeclareMathOperator{\diag}{diag}

\theoremstyle{plain}
\newtheorem{theorem}{Theorem}[section]
\newtheorem{proposition}[theorem]{Proposition}
\newtheorem{corollary}[theorem]{Corollary}
\theoremstyle{remark}
\newtheorem{remark}[theorem]{Remark}

\numberwithin{equation}{section}
\title{Cone Rayleigh Levels: Finite Perturbations and Certified Control}

\author{Yavdat S. Il'yasov$^{*}$ and Nur F. Valeev\\[0.6em]
\small Institute of Mathematics with Computing Centre,\\
\small Ufa Federal Research Centre of the Russian Academy of Sciences,\\
\small 112 Chernyshevsky St., Ufa 450008, Russia\\[0.4em]
\small \textsuperscript{*}Corresponding author: \href{mailto:ilyasov02@gmail.com}{ilyasov02@gmail.com}}
\date{}

\begin{document}
\maketitle

\begin{abstract}
	We study the reuse of positive trial profiles under finite
	perturbations of nonsymmetric matrix pencils $B-\lambda G$.
	The lower and upper cone Rayleigh levels need not coincide
	and are defined without requiring positive eigenvectors.
	In the positive orthant with positive diagonal $G$,
	we derive computable perturbation bounds and determine
	the exact worst-case trial gap over prescribed independent
	entrywise perturbation classes. This yields the largest
	uniform radius meeting a given trial-gap tolerance for
	fixed profiles, certifying trial-value accuracy without
	recomputation.
	Experiments on a nonnegative operator and five signed
	matrices compare sufficient and optimal radii, directional
	thresholds, and cold- and warm-start recomputation.
	Several signed cases exhibit severely limited uniform
	profile reuse despite the optimality of the radius.
	Exact rational checks verify the reported bounds and
	worst-case constructions for stored numerical inputs.
	A finite-budget linear program and differentiable constraints
	illustrate applications to verified control and learning.
\end{abstract}



% ================================================================
\section{Introduction}
\label{sec:introduction}
% ================================================================

For a changing nonsymmetric operator, recomputing variational
trial profiles after every update can be expensive.
We study when previously computed profiles retain a prescribed
trial-value accuracy and how their bounds certify the effect
of a constrained intervention.
The main result determines the exact worst-case trial gap
over an entrywise perturbation class and the largest uniform
radius satisfying a given gap tolerance at fixed profiles.
Finite-budget control and differentiable constraints provide
applications of these results.

We consider real matrix pencils
\begin{equation}
	\label{eq:matrix-pencil}
	\mathcal P_\theta(\lambda)=B_\theta-\lambda G,
	\qquad B_\theta,G\in\R^{N\times N},
\end{equation}
with fixed $G$.
Under the assumptions stated in
Section~\ref{sec:cone-RQ-minimax}, the cone variational framework
defines lower and upper levels,
$\lambdalow_C(B_\theta,G)$ and
$\lambdaup_C(B_\theta,G)$.
The framework accommodates nonsymmetric operators that
do not preserve the cone and treats the two levels separately.

The componentwise estimates and linear programming formulations
specialize to the positive orthant and
$G=\diag(g_1,\ldots,g_N)$, $g_i>0$.
Positive trial profiles $u,v$ give
\[
\ell_B(u)=\min_i\frac{(Bu)_i}{g_i u_i},
\qquad
q_B(v)=\max_j\frac{(B^Tv)_j}{g_jv_j},
\]
with
\[
\ell_B(u)
\le \lambdalow_C(B,G)
\le \lambdaup_C(B,G)
\le q_B(v).
\]
The trial gap $D_B(u,v)=q_B(v)-\ell_B(u)$ bounds each
trial-value error and the separation of the levels.
Throughout, profile accuracy means accuracy of these
trial values.
For irreducible nonnegative $B$, both levels equal
the Perron eigenvalue of $G^{-1}B$.
Under signed perturbations, the cone levels remain defined
even after positive eigenvectors are lost, and may separate.
They therefore extend the variational targets beyond
the positive eigenpair regime.

For a finite perturbation $B\mapsto B+E$, retained profiles
satisfy the computable estimate
\[
D_{B+E}(u,v)
\le \Delta_E(u,v):=D_B(u,v)+b_+-a_-,
\]
where
\[
a_-=\min_i\frac{(Eu)_i}{g_i u_i},
\qquad
b_+=\max_j\frac{(E^Tv)_j}{g_jv_j}.
\]
Its evaluation requires only $Eu$ and $E^Tv$,
allowing inexpensive checks for local or low-rank updates.
For a nonnegative mask $M$, we determine the exact maximum
of $D_{B+E}(u,v)$ over all independent entrywise perturbations
satisfying $|E|\le rM$.
The resulting optimal uniform trial-gap radius accounts
for the baseline componentwise margins and the shared entries
in the forward and transpose perturbation products.
It separates the loss in a simpler sufficient radius
from variation in robustness across perturbation directions.

The variational framework builds on several lines of earlier work.
In the positive regime, the ratio bounds connect to
Perron--Frobenius and Collatz--Wielandt theory
\cite{Berman,BirkhoffVarga1958,Friedland1990}.
The broader approach developed through minimax bifurcation
formulas \cite{Ilyasov2001,IlyasFunc2007,Ilyasov2026},
numerical realizations
\cite{IvanIlya,IlIvan1,SalazarIlaysov2020},
and nonlinear extensions \cite{ilyasov2021,IlyasJDE24}.
Cone minimax theory for arbitrary real matrices was developed
in \cite{IlyasovValeev2024}.
Related questions for Galerkin approximations of non-selfadjoint
operator pencils are studied in \cite{IlyasovValeevPencils2026}.
We give a self-contained proof of the finite-dimensional
minimax statement under the stated assumptions on $G$
and the cone.

Classical spectral perturbation theory estimates changes in
eigenvalues and spectral sets under matrix perturbations
\cite{BauerFike1960,Elsner1985}.
In network intervention, spectral-radius reduction by edge
removal and leading-eigenvalue optimization are established
objectives \cite{VanMieghemEtAl2011,ChenEtAl2016}.
We address the certification of previously computed trial
profiles after a finite operator change.
For fixed positive profiles, we determine the exact worst-case
trial gap over an independent entrywise perturbation class
and the largest radius for which this gap remains within
a prescribed tolerance.
In the nonnegative irreducible regime, the resulting enclosure
bounds the Perron root; for signed operators, it bounds the
two cone levels separately.
The certification problem therefore concerns the continued
validity of retained variational bounds and their use in
verifying finite interventions.
The results are organized into three components.
\begin{enumerate}
	\item
	\textbf{Finite perturbation bounds and optimal uniform radii.}
	We derive direction-dependent bounds, characterize equality
	in the finite-gap estimate, and determine the exact worst-case
	gap over independent entrywise perturbations.
	This yields the largest uniform trial-gap radius for fixed
	profiles and a prescribed tolerance.
	An irreducible signed example with exactly separated levels
	illustrates the distinction between trial-value errors
	and intrinsic level separation.
	
	\item
	\textbf{Computational verification of profile reuse.}
	Local and rank-two perturbation directions on a nonnegative
	reference operator and five signed matrices compare
	sufficient and optimal uniform radii with directional
	thresholds.
	Exact rational worst-case constructions verify sharpness
	of the optimal radii.
	Cold- and warm-start recomputation supplies refined
	enclosures and a cost comparison, with smaller median
	recomputation times under warm starts in all tested cases.
	
	\item
	\textbf{Verified finite-budget control and learning constraints.}
	A linear program minimizes a fixed-profile upper bound
	under an $L^1$ budget and entrywise capacities.
	An independent trial-profile test certifies finite decrease.
	A four-method comparison applies common post-intervention
	profile refinement and obtains positive exact rational
	decrease certificates for all $120$ interventions.
	Smooth bounds with explicit approximation errors support
	differentiable constraints and learning applications.
\end{enumerate}

Directed graph operators provide the application setting.
For a nonnegative adjacency matrix $A=(a_{ij})$, we use
\begin{equation}
	\label{eq:edge-convention}
	a_{ij}>0\quad\Longleftrightarrow\quad j\to i.
\end{equation}
Graph-supported models satisfy
\begin{equation}
	\label{eq:sparsity-support}
	a_{ij}=0\quad\Longrightarrow\quad b_{ij}(\theta)=0,
	\qquad i\neq j,
\end{equation}
with additional terms specified separately.
Related approaches use directed Laplacians, convolutions,
magnetic operators, and separate incoming and outgoing
propagation
\cite{MaEtAl2019,TongEtAl2020,ZhangMagNet2021,
	PerlmutterEtAl2023,RossiEtAl2024};
Rayleigh quotients also appear in learning objectives
\cite{DongZhangWangRQGNN}.
Here, cone bounds constrain the nonsymmetric propagation
operator directly.
The SBM experiments satisfy the prescribed cap numerically
in all $50$ runs.
For the tested Cora model, adaptive Perron control reduces
the spectral level while preserving the reported test
accuracy without retraining.

The paper first presents cone levels and trial-profile bounds,
then develops finite perturbation estimates and optimal
uniform radii.
The subsequent sections address finite-budget control,
differentiable constraints and learning, and numerical
verification.
Additional intervention comparisons and refined certificates
are presented in Appendix~\ref{sec:additional-control-comparisons}.

% ================================================================
\section{Cone levels and trial-profile bounds}
\label{sec:cone-RQ-minimax}
% ================================================================

We recall the cone variational framework
\cite{IlyasovValeev2024} and give a self-contained proof
of the matrix-pencil formulation used below.
No cone-preservation assumption is imposed on $B$.

% ----------------------------------------------------------------
\subsection{Cone levels and their coincidence}
\label{subsec:cone-extended-RQ}
\label{subsec:main-minimax-theorem}
% ----------------------------------------------------------------

Let $\overline C\subset\R^N$ be a closed, convex, solid,
self-dual cone, and write
\[
C=\overline C\setminus\{0\},
\qquad
C^\circ=\operatorname{int}\overline C.
\]
For $B,G\in\R^{N\times N}$, assume
\begin{equation}
	\label{eq:G-admissibility}
	\langle Gu,v\rangle>0
	\quad\text{for}\quad
	(u,v)\in(C\times C^\circ)\cup(C^\circ\times C).
\end{equation}
On these mixed domains, define
\begin{equation}
	\label{eq:cone-extended-RQ}
	R_B(u,v)
	=\frac{\langle Bu,v\rangle}{\langle Gu,v\rangle}.
\end{equation}
The cone levels are
\begin{align}
	\lambdaup_C(B,G)
	&:=\sup_{u\in C}\inf_{v\in C^\circ}R_B(u,v),
	\label{eq:upper-quasi-eigenvalue}\\
	\lambdalow_C(B,G)
	&:=\inf_{v\in C}\sup_{u\in C^\circ}R_B(u,v).
	\label{eq:lower-quasi-eigenvalue}
\end{align}
The terms ``lower'' and ``upper'' refer to the ordering
proved below, rather than to the order of the supremum
and infimum in the mixed-domain definitions.

\begin{theorem}[Cone minimax principle]
	\label{thm:main-minimax}
	Under \eqref{eq:G-admissibility}, both levels are finite.
	The outer supremum defining $\lambdaup_C(B,G)$
	is attained at some $u_C\in C$, and the outer infimum
	defining $\lambdalow_C(B,G)$ is attained at some $v_C\in C$.
	We call any such pair $(u_C,v_C)$ an extremizing pair.
	Moreover,
	\begin{align}
		\lambdalow_C(B,G)
		&=\sup_{u\in C^\circ}\inf_{v\in C^\circ}R_B(u,v),
		\label{eq:lower-interior-level}\\
		\lambdaup_C(B,G)
		&=\inf_{v\in C^\circ}\sup_{u\in C^\circ}R_B(u,v),
		\label{eq:upper-interior-level}
	\end{align}
	and
	\[
	\lambdalow_C(B,G)\le\lambdaup_C(B,G).
	\]
	If an extremizing pair satisfies
	$\langle Gu_C,v_C\rangle>0$, the levels coincide.
	In particular, either $u_C\in C^\circ$ or
	$v_C\in C^\circ$ is sufficient for coincidence.
	
	If the levels coincide, write
	$\lambda_C=\lambdalow_C(B,G)=\lambdaup_C(B,G)$.
	Then every extremizing pair satisfies
	\begin{equation}
		\label{eq:cone-complementarity}
		\begin{gathered}
			r_C:=Bu_C-\lambda_CGu_C\in\overline C,
			\qquad
			s_C:=\lambda_CG^Tv_C-B^Tv_C\in\overline C,
			\\
			\langle r_C,v_C\rangle
			=\langle s_C,u_C\rangle=0.
		\end{gathered}
	\end{equation}
	Consequently, $v_C\in C^\circ$ implies $r_C=0$,
	and $u_C\in C^\circ$ implies $s_C=0$.
\end{theorem}

\begin{proof}
	Fix $e\in C^\circ$.
	By homogeneity, both variables can be restricted
	to the compact convex base
	\[
	\mathcal B
	=\{x\in\overline C:\langle x,e\rangle=1\}.
	\]
	The set $\mathcal B^\circ:=\mathcal B\cap C^\circ$
	is convex and dense in $\mathcal B$.
	On each mixed domain
	$\mathcal B\times\mathcal B^\circ$ and
	$\mathcal B^\circ\times\mathcal B$,
	the quotient is separately continuous, quasiconvex,
	and quasiconcave in each variable.
	
	Sion's minimax theorem \cite{Sion1958} therefore gives
	\[
	\begin{aligned}
		\sup_{u\in\mathcal B}\inf_{v\in\mathcal B^\circ}R_B(u,v)
		&=
		\inf_{v\in\mathcal B^\circ}\sup_{u\in\mathcal B}R_B(u,v),\\
		\inf_{v\in\mathcal B}\sup_{u\in\mathcal B^\circ}R_B(u,v)
		&=
		\sup_{u\in\mathcal B^\circ}\inf_{v\in\mathcal B}R_B(u,v).
	\end{aligned}
	\]
	For a fixed interior argument, continuity and density
	allow the other inner supremum or infimum over
	$\mathcal B$ to be taken over $\mathcal B^\circ$.
	This proves
	\eqref{eq:lower-interior-level}--%
	\eqref{eq:upper-interior-level}.
	The usual minimax inequality on the interior domain
	then yields $\alpha\le\beta$, where
	\[
	\alpha=\lambdalow_C(B,G),
	\qquad
	\beta=\lambdaup_C(B,G).
	\]
	
	The functions $R_B(e,\cdot)$ and $R_B(\cdot,e)$
	are continuous on $\mathcal B$.
	Consequently,
\[
\begin{aligned}
	-\infty
	&<\min_{v\in\mathcal B}R_B(e,v)\le\alpha,\\
	\alpha
	\le\beta
	\le&\max_{u\in\mathcal B}R_B(u,e)<+\infty.
\end{aligned}
\]
	The function
	$u\mapsto\inf_{v\in\mathcal B^\circ}R_B(u,v)$
	is upper semicontinuous, while
	$v\mapsto\sup_{u\in\mathcal B^\circ}R_B(u,v)$
	is lower semicontinuous.
	Compactness gives attainment of the corresponding
	outer extrema.
	
	The extremizer inequalities, density, and self-duality
	imply
	\[
	Bu_C-\beta Gu_C\in\overline C,
	\qquad
	\alpha G^Tv_C-B^Tv_C\in\overline C.
	\]
	Pairing with $v_C$ and $u_C$, respectively, gives
	\[
	\begin{aligned}
		0
		&\le
		\langle Bu_C-\beta Gu_C,v_C\rangle
		+\langle\alpha G^Tv_C-B^Tv_C,u_C\rangle\\
		&=(\alpha-\beta)\langle Gu_C,v_C\rangle.
	\end{aligned}
	\]
	If $\langle Gu_C,v_C\rangle>0$, then $\alpha\le\beta$
	forces equality.
	At a common level, the two nonnegative pairings sum
	to zero, proving \eqref{eq:cone-complementarity}.
	Finally, an interior vector pairs positively with
	every nonzero element of $\overline C$, which proves
	the residual assertions.
\end{proof}

We call $u_C$ and $v_C$ right and left quasi-eigenvectors.
If $v_C\in C^\circ$, then
$Bu_C=\lambda_CGu_C$;
if $u_C\in C^\circ$, then
$B^Tv_C=\lambda_CG^Tv_C$.
Thus interiority of either extremizer implies coincidence
of the levels and singularity of $B-\lambda_CG$.
For a regular pencil, the common level is therefore
a generalized eigenvalue.
If both extremizers are interior, they form a generalized
right--left eigenpair.
Coincidence alone does not imply that both quasi-eigenvectors
are eigenvectors.
The complementarity relations remain valid when
$\langle Gu_C,v_C\rangle=0$, without evaluating the quotient
at that boundary pair.

% ----------------------------------------------------------------
\subsection{Trial-profile bounds in the positive orthant}
\label{subsec:positive-orthant-formulas}
% ----------------------------------------------------------------

Let $\overline C=\R_+^N$ and
$G=\operatorname{diag}(g_i)$ with $g_i>0$.
For strictly positive trial profiles $u,v$, define
\begin{equation}
	\label{eq:trial-certificates}
	\underline{\lambda}_B(u)
	:=\min_i\frac{(Bu)_i}{g_i u_i},
	\qquad
	\overline{\lambda}_B(v)
	:=\max_j\frac{(B^Tv)_j}{g_jv_j}.
\end{equation}
For fixed $u>0$, the quotient is a weighted average
of the ratios $(Bu)_i/(g_i u_i)$; the analogous
representation holds for fixed $v>0$.
Thus the interior formulas yield
\begin{align}
	\lambdalow_C(B,G)
	&=\sup_{u>0}\underline{\lambda}_B(u),
	\label{eq:lower-CW-general}\\
	\lambdaup_C(B,G)
	&=\inf_{v>0}\overline{\lambda}_B(v).
	\label{eq:upper-CW-general}
\end{align}
Every positive pair therefore gives
\begin{equation}
	\label{eq:trial-cone-level-enclosure}
	\underline{\lambda}_B(u)
	\le\lambdalow_C(B,G)
	\le\lambdaup_C(B,G)
	\le\overline{\lambda}_B(v).
\end{equation}
These bounds require only forward and transpose
matrix--vector products, without interior extremizers
or coincident levels.

The trial gap
\begin{equation}
	\label{eq:trial-profile-gap}
	\delta_B(u,v)
	:=\overline{\lambda}_B(v)-\underline{\lambda}_B(u)
	\ge0
\end{equation}
bounds both trial-value errors and the level separation:
\[
\begin{aligned}
	0&\le\lambdalow_C(B,G)-\underline{\lambda}_B(u)
	\le\delta_B(u,v),\\
	0&\le\overline{\lambda}_B(v)-\lambdaup_C(B,G)
	\le\delta_B(u,v),\\
	0&\le\lambdaup_C(B,G)-\lambdalow_C(B,G)
	\le\delta_B(u,v).
\end{aligned}
\]
A small gap certifies accuracy of the trial values,
but gives no bound on the distance of the profiles
to particular quasi-eigenvectors.

The levels also admit the feasibility characterizations
\begin{align}
	\lambdalow_C(B,G)
	&=
	\sup\{\lambda\in\R:\exists u>0,\ Bu\ge\lambda Gu\},
	\label{eq:lower-level-feasibility-threshold}\\
	\lambdaup_C(B,G)
	&=
	\inf\{\lambda\in\R:\exists v>0,\ B^Tv\le\lambda Gv\}.
	\label{eq:upper-level-feasibility-threshold}
\end{align}
Every value strictly below the lower threshold is feasible
in the first problem, and every value strictly above
the upper threshold is feasible in the second.
Feasibility at either threshold is not guaranteed.

For irreducible $B\ge0$, both levels equal $\rho(G^{-1}B)$.
For irreducible Metzler $B$, they equal the spectral abscissa
of $G^{-1}B$.
For general signed matrices,
\eqref{eq:trial-cone-level-enclosure} bounds the cone levels,
not necessarily either of these spectral quantities.
All stated bounds concern exact arithmetic;
rigorous numerical certification requires control
of rounding errors.

% ================================================================
\section{Finite perturbations and preservation of profile accuracy}
\label{sec:finite-profile-accuracy}
% ================================================================

We now specialize to
$\overline C=\mathbb R_+^N$ and
$G=\operatorname{diag}(g_1,\ldots,g_N)$ with $g_i>0$.
For brevity, write the trial bounds from
\eqref{eq:trial-certificates} as
\[
\ell_A(u):=\underline{\lambda}_A(u)
=\min_i\frac{(Au)_i}{g_i u_i},
\qquad
q_A(v):=\overline{\lambda}_A(v)
=\max_j\frac{(A^Tv)_j}{g_j v_j}.
\]
Thus, for every real matrix $A$ and $u,v>0$,
\begin{equation}
	\label{eq:profile-enclosure}
	\ell_A(u)
	\le\lambdalow_C(A,G)
	\le\lambdaup_C(A,G)
	\le q_A(v).
\end{equation}
We study the accuracy of these bounds when profiles
computed for $B$ are retained after a finite perturbation
$B\mapsto B+E$.
Here profile accuracy refers to the trial values,
not to distances between vectors.
Neither $B$ nor $B+E$ is assumed to preserve the cone.
We first obtain inexpensive bounds for individual perturbations,
then determine the exact worst-case trial gap and the largest
uniform trial-gap radius for an entrywise perturbation class.

% ----------------------------------------------------------------
\subsection{Finite perturbation bounds for retained profiles}
% ----------------------------------------------------------------

Fix $u,v>0$ and set
\[
s_i=\frac{(Bu)_i}{g_i u_i},
\qquad
t_j=\frac{(B^Tv)_j}{g_j v_j},
\qquad
D_B(u,v):=q_B(v)-\ell_B(u).
\]
Thus $D_B(u,v)$ is the gap
$\delta_B(u,v)$ introduced in
\eqref{eq:trial-profile-gap}.
For a real perturbation $E$, define
\[
a_i(E;u)=\frac{(Eu)_i}{g_i u_i},
\qquad
b_j(E;v)=\frac{(E^Tv)_j}{g_j v_j},
\]
and
\[
a_-(E;u)=\min_i a_i(E;u),
\qquad
b_+(E;v)=\max_j b_j(E;v).
\]
The updated bounds satisfy the exact identities
\begin{equation}
	\label{eq:updated-profile-bounds}
	\ell_{B+E}(u)=\min_i\{s_i+a_i(E;u)\},
	\qquad
	q_{B+E}(v)=\max_j\{t_j+b_j(E;v)\}.
\end{equation}

\begin{theorem}[Finite perturbation enclosure]
	\label{thm:finite-profile-enclosure}
	For every $B,E\in\mathbb R^{N\times N}$ and $u,v>0$,
	\begin{equation}
		\label{eq:finite-profile-enclosure}
		\begin{aligned}
			\ell_B(u)+a_-(E;u)
			&\le \ell_{B+E}(u)
			\le \lambdalow_C(B+E,G)\\
			&\le \lambdaup_C(B+E,G)
			\le q_{B+E}(v)
			\le q_B(v)+b_+(E;v).
		\end{aligned}
	\end{equation}
	In particular, define
	\[
	\Delta_E(u,v)
	:=D_B(u,v)+b_+(E;v)-a_-(E;u).
	\]
	Then
	\[
	0\le D_{B+E}(u,v)\le\Delta_E(u,v),
	\]
	and
	\begin{align}
		0
		&\le\lambdaup_C(B+E,G)-\lambdalow_C(B+E,G)
		\le D_{B+E}(u,v)\le\Delta_E(u,v),
		\label{eq:finite-level-separation}\\
		0
		&\le\lambdalow_C(B+E,G)-\ell_{B+E}(u)
		\le D_{B+E}(u,v)\le\Delta_E(u,v),
		\label{eq:finite-lower-profile-error}\\
		0
		&\le q_{B+E}(v)-\lambdaup_C(B+E,G)
		\le D_{B+E}(u,v)\le\Delta_E(u,v).
		\label{eq:finite-upper-profile-error}
	\end{align}
\end{theorem}

\begin{proof}
	By \eqref{eq:updated-profile-bounds},
	\[
	\ell_{B+E}(u)\ge\ell_B(u)+a_-(E;u),
	\qquad
	q_{B+E}(v)\le q_B(v)+b_+(E;v).
	\]
	Combining these inequalities with
	\eqref{eq:profile-enclosure} proves the enclosure
	and the bound on $D_{B+E}(u,v)$.
	Each remaining nonnegative difference is bounded
	by this updated trial gap.
\end{proof}

For a prescribed tolerance $\varepsilon>0$,
\begin{equation}
	\label{eq:profile-accuracy-preservation}
	\Delta_E(u,v)\le\varepsilon
\end{equation}
therefore guarantees that both updated trial values
approximate their respective cone levels with absolute
errors at most $\varepsilon$.
If the initial ratios $s_i,t_j$ are retained, the sharper
sufficient test
\begin{equation}
	\label{eq:exact-updated-gap-test}
	D_{B+E}(u,v)
	=
	\max_j\{t_j+b_j(E;v)\}
	-\min_i\{s_i+a_i(E;u)\}
	\le\varepsilon
\end{equation}
requires no new products with $B$ or $B^T$.
Neither test requires new extremizers.

\begin{remark}[Structured perturbations and evaluation cost]
	\label{rem:structured-perturbation-cost}
	After the baseline ratios have been stored, the update
	requires only $Eu$, $E^Tv$, and componentwise operations.
	If $E=XY^T$ with $X,Y\in\mathbb R^{N\times r}$,
	these products are $X(Y^Tu)$ and $Y(X^Tv)$, giving
	an $O(Nr)$ evaluation cost without forming $E$.
	For sparse $E$, a full update costs
	$O(\operatorname{nnz}(E)+N)$.
	The coarser quantities $a_-$ and $b_+$ can be computed
	using only the affected rows and columns, including
	zero ratios for unaffected components.
	These costs concern certificate evaluation;
	the cost of recomputing profiles depends on the solver.
\end{remark}

% ----------------------------------------------------------------
\subsection{Exactness of the direction-dependent estimate}
% ----------------------------------------------------------------

The direction-dependent estimate can be exact even when
the trial profiles are not eigenvectors. The following
criterion characterizes equality through the active
componentwise ratios.

\begin{proposition}[Exactness of the finite-gap estimate]
	\label{prop:finite-gap-exactness}
	\label{cor:finite-interior-pair}
	For every $B,E\in\mathbb R^{N\times N}$ and $u,v>0$,
	\[
	D_{B+E}(u,v)=\Delta_E(u,v)
	\]
	if and only if
	\begin{equation}
		\label{eq:finite-gap-exactness-criterion}
		\begin{gathered}
			\arg\min_i s_i\cap\arg\min_i a_i(E;u)
			\ne\varnothing,\\
			\arg\max_j t_j\cap\arg\max_j b_j(E;v)
			\ne\varnothing.
		\end{gathered}
	\end{equation}
	In particular, suppose $u_*,v_*>0$ satisfy
	\[
	Bu_*=\lambda Gu_*,
	\qquad B^Tv_*=\lambda Gv_*.
	\]
	Then both cone levels of $(B,G)$ equal $\lambda$.
	For every real perturbation $E$,
	\begin{equation}
		\label{eq:interior-pair-perturbation}
		\lambda+a_-(E;u_*)
		\le\lambdalow_C(B+E,G)
		\le\lambdaup_C(B+E,G)
		\le\lambda+b_+(E;v_*),
	\end{equation}
	and
	\begin{equation}
		\label{eq:exact-pair-updated-gap}
		D_{B+E}(u_*,v_*)
		=\Delta_E(u_*,v_*)
		=b_+(E;v_*)-a_-(E;u_*).
	\end{equation}
\end{proposition}

\begin{proof}
	Write $a_i=a_i(E;u)$ and $b_j=b_j(E;v)$.
	By \eqref{eq:updated-profile-bounds},
	\[
	\begin{aligned}
		\Delta_E-D_{B+E}
		={}&\bigl[\max_j t_j+\max_j b_j
		-\max_j(t_j+b_j)\bigr]\\
		&+\bigl[\min_i(s_i+a_i)
		-\min_i s_i-\min_i a_i\bigr].
	\end{aligned}
	\]
	Both terms are nonnegative. The first vanishes exactly
	when $t_j$ and $b_j$ have a common maximizing index;
	the second vanishes exactly when $s_i$ and $a_i$ have
	a common minimizing index. This proves
	\eqref{eq:finite-gap-exactness-criterion}.
	For an exact positive right--left eigenpair, all baseline
	ratios equal $\lambda$, so both intersection conditions
	hold for every $E$ and $D_B(u_*,v_*)=0$.
	The baseline coincidence and the perturbed enclosure
	follow from \eqref{eq:profile-enclosure} and
	Theorem~\ref{thm:finite-profile-enclosure}.
\end{proof}

Exactness here concerns the retained-profile gap;
it does not imply that the retained bounds attain the
perturbed cone levels. As before, this gap bounds both
trial-value errors and the level separation.

No simplicity assumption is needed.
The perturbed levels need not coincide or be eigenvalues.
For a fixed direction $H$ and $t\ge0$, put
\[
\omega(H):=b_+(H;v_*)-a_-(H;u_*)\ge0.
\]
Then $D_{B+tH}(u_*,v_*)=t\omega(H)$.
If $\omega(H)>0$, both trial-value errors are at most
$\varepsilon$ for
$0\le t\le\varepsilon/\omega(H)$.
If $\omega(H)=0$, the retained bounds coincide for
every $t\ge0$.
In particular, if
\[
Eu_*=cGu_*,
\qquad
E^Tv_*=cGv_*,
\]
the original pair remains exact and both perturbed
levels equal $\lambda+c$; this includes $E=cG$.

% ----------------------------------------------------------------
\subsection{Exact worst-case gaps and optimal uniform radii}
\label{subsec:optimal-uniform-profile-radius}
% ----------------------------------------------------------------

Fix $u,v>0$ and a nonnegative mask $M$.
For $r\ge0$, consider the entrywise perturbation class
\[
\mathcal U_M(r)
:=\{E\in\mathbb R^{N\times N}:|E_{kl}|\le rM_{kl}
\text{ for all }k,l\}.
\]
The entries vary independently within these bounds.
Set
\[
d_{ij}:=t_j-s_i,\qquad
\alpha_i:=\frac{(Mu)_i}{g_i u_i},\qquad
\beta_j:=\frac{(M^Tv)_j}{g_jv_j},
\]
and define
\begin{equation}
	\label{eq:mask-pair-coefficient}
	\kappa_{ij}
	:=\alpha_i+\beta_j
	-2M_{ij}\min\left\{
	\frac{u_j}{g_i u_i},\frac{v_i}{g_jv_j}
	\right\}.
\end{equation}
These coefficients account for the fact that the same entry
$E_{ij}$ contributes with opposite signs to
$b_j(E;v)-a_i(E;u)$.

Here the robust quantity is the gap between the two retained
trial bounds. Its exact maximization over the mask class
determines a uniform profile-reuse guarantee.

\begin{theorem}[Exact worst-case gap and optimal uniform radius]
	\label{thm:optimal-uniform-profile-radius}
	For every $r\ge0$, all $\kappa_{ij}$ are nonnegative and
	\begin{equation}
		\label{eq:exact-mask-worst-gap}
		\max_{E\in\mathcal U_M(r)}D_{B+E}(u,v)
		=\max_{i,j}\{d_{ij}+r\kappa_{ij}\}.
	\end{equation}
	Suppose $D_B(u,v)\le\varepsilon$ and define
	\begin{equation}
		\label{eq:optimal-uniform-profile-radius}
		r_*:=\min_{\kappa_{ij}>0}
		\frac{\varepsilon-d_{ij}}{\kappa_{ij}},
	\end{equation}
	where the minimum over the empty set is $+\infty$.
	Then
	\[
	D_{B+E}(u,v)\le\varepsilon
	\quad\text{for every }E\in\mathcal U_M(r)
	\quad\Longleftrightarrow\quad
	0\le r\le r_*.
	\]
	In particular, if $r_*<\infty$, every $r>r_*$ admits
	an $E\in\mathcal U_M(r)$ with
	$D_{B+E}(u,v)>\varepsilon$.
	Thus $r_*$ is the largest uniform radius for the prescribed
	trial-gap tolerance at the fixed profiles $u,v$ over the
	full class $\mathcal U_M(r)$.
\end{theorem}

\begin{proof}
	The exact updated gap can be written as
	\[
	D_{B+E}(u,v)
	=\max_{i,j}\{d_{ij}+b_j(E;v)-a_i(E;u)\}.
	\]
	For each pair $(i,j)$, write
	\[
	b_j(E;v)-a_i(E;u)
	=\sum_{k,l}c^{ij}_{kl}E_{kl},\qquad
	c^{ij}_{kl}
	=\mathbf1_{\{l=j\}}\frac{v_k}{g_jv_j}
	-\mathbf1_{\{k=i\}}\frac{u_l}{g_i u_i}.
	\]
	Independent variation of the entries gives
	\[
	\max_{E\in\mathcal U_M(r)}
	\sum_{k,l}c^{ij}_{kl}E_{kl}
	=r\sum_{k,l}M_{kl}|c^{ij}_{kl}|
	=r\kappa_{ij}.
	\]
	The last equality follows by separating the overlapping
	entry $(i,j)$ and using $|x-y|=x+y-2\min\{x,y\}$
	for $x,y\ge0$. It also proves $\kappa_{ij}\ge0$.
	Interchanging the two maxima proves
	\eqref{eq:exact-mask-worst-gap}.
	For any maximizing pair $(i,j)$, the maximum is attained by
	\[
	E_{kl}=rM_{kl}\operatorname{sgn}(c^{ij}_{kl}),
	\qquad \operatorname{sgn}(0)=0.
	\]
	
	Since $d_{ij}\le D_B(u,v)\le\varepsilon$, pairs with
	$\kappa_{ij}=0$ impose no restriction on $r$.
	For each remaining pair, the inequality
	$d_{ij}+r\kappa_{ij}\le\varepsilon$ is equivalent to
	$r\le(\varepsilon-d_{ij})/\kappa_{ij}$.
	Taking the minimum proves the radius formula and its optimality.
\end{proof}

The theorem is exact for the full class $\mathcal U_M(r)$.
Additional sign, rank, or coupling constraints on $E$ may
increase the admissible radius; $r_*$ remains sufficient
but need not be optimal for such restricted classes.
Its optimality concerns the trial gap. Exceeding $r_*$
does not imply that every admissible perturbation violates
the tolerance, or that an individual trial-value error
exceeds $\varepsilon$.

A simpler sufficient radius follows from the
individual-perturbation bound $\Delta_E$.
Write
\[
\alpha_M(u):=\max_i\alpha_i,\qquad
\beta_M(v):=\max_j\beta_j.
\]

\begin{corollary}[A readily computable sufficient radius]
	\label{cor:uniform-profile-radius}
	For every $E\in\mathcal U_M(r)$,
	\[
	D_{B+E}(u,v)\le\Delta_E(u,v)
	\le D_B(u,v)+r\bigl(\alpha_M(u)+\beta_M(v)\bigr).
	\]
	If $D_B(u,v)\le\varepsilon$ and
	$\alpha_M(u)+\beta_M(v)>0$, then
	\begin{equation}
		\label{eq:sufficient-uniform-profile-radius}
		r_{\rm cert}
		:=\frac{\varepsilon-D_B(u,v)}
		{\alpha_M(u)+\beta_M(v)}
		\le r_*.
	\end{equation}
	Every $0\le r\le r_{\rm cert}$ therefore guarantees
	that both trial-value errors and the level separation
	are at most $\varepsilon$ throughout $\mathcal U_M(r)$.
\end{corollary}

\begin{proof}
	The entrywise inequalities $Eu\ge-rMu$ and
	$E^Tv\le rM^Tv$ imply
	$a_-(E;u)\ge-r\alpha_M(u)$ and
	$b_+(E;v)\le r\beta_M(v)$.
	Theorem~\ref{thm:finite-profile-enclosure} gives the stated
	bound, and the maximality of $r_*$ gives
	\eqref{eq:sufficient-uniform-profile-radius}.
\end{proof}

If $\alpha_M(u)+\beta_M(v)=0$, then $M=0$ and
$r_*=+\infty$ whenever the baseline gap meets the tolerance.
Neither radius requires $B+E\ge0$.
Computing $r_{\rm cert}$ requires two products with $M$
and componentwise operations. Direct evaluation of $r_*$
additionally checks $N^2$ index pairs, without recomputing
profiles or solving a cone-level problem.

The exact formula identifies two sources of conservatism
in $r_{\rm cert}$: replacing the separate baseline margins
$\varepsilon-d_{ij}$ by $\varepsilon-D_B$, and replacing
$\kappa_{ij}$ by $\alpha_M+\beta_M$.
A comparison with a single perturbation direction also
includes the difference between that direction and the
worst case over the full mask class. Consequently, a large
observed directional threshold alone does not measure
the loss in the sufficient radius.

\begin{corollary}[Frobenius-ball certificate]
	\label{cor:frobenius-profile-radius}
	For fixed positive profiles $u,v$, set
	\[
	p^{(j)}:=\frac{v}{g_jv_j},
	\qquad
	q^{(i)}:=\frac{u}{g_iu_i},
	\qquad
	C^{ij}:=p^{(j)}e_j^T-e_i(q^{(i)})^T,
	\]
	and define
	\[
	\nu_{ij}:=\|C^{ij}\|_F
	=\left(
	\|p^{(j)}\|_2^2+\|q^{(i)}\|_2^2
	-2p_i^{(j)}q_j^{(i)}
	\right)^{1/2}.
	\]
	Then, for every $r\ge0$,
	\[
	\max_{\|E\|_F\le r}D_{B+E}(u,v)
	=\max_{i,j}\{d_{ij}+r\nu_{ij}\}.
	\]
	If $D_B(u,v)\le\varepsilon$, the largest uniform radius
	for which $D_{B+E}(u,v)\le\varepsilon$ throughout
	the Frobenius ball is
	\[
	r_{*,F}:=\min_{\nu_{ij}>0}
	\frac{\varepsilon-d_{ij}}{\nu_{ij}},
	\]
	where the minimum over the empty set is $+\infty$.
	The same maximum and radius hold if the perturbations
	are additionally restricted by $\operatorname{rank}E\le2$.
\end{corollary}

\begin{proof}
	The trial gap is
	\[
	D_{B+E}(u,v)
	=\max_{i,j}\{d_{ij}+\langle C^{ij},E\rangle_F\}.
	\]
	By Frobenius-norm duality, maximization over the ball
	gives $d_{ij}+r\nu_{ij}$ for each pair.
	If $\nu_{ij}>0$, this value is attained at
	$E=rC^{ij}/\nu_{ij}$, whose rank is at most two;
	if $\nu_{ij}=0$, take $E=0$.
	Interchanging the maxima proves both maximum formulas.
	The radius follows by imposing
	$d_{ij}+r\nu_{ij}\le\varepsilon$ for every pair.
\end{proof}

\paragraph{Relation to spectral perturbation bounds.}
Bauer--Fike and Elsner bounds
\cite{BauerFike1960,Elsner1985} control spectral displacement,
whereas $r_{\rm cert}$ and $r_*$ control the retained trial gap.
For example, adding $tI$ when $G=I$ shifts both trial bounds
and every eigenvalue by $t$, leaving the trial gap unchanged.
For a simple eigenvalue, first-order sensitivity describes
local displacement but gives no finite-radius guarantee
without a remainder bound \cite{GreenbaumLiOverton2020}.
Thus spectral and trial-gap radii concern different tolerances;
Table~\ref{tab:optimal-profile-radii} compares radii for the
same fixed-profile task.

For an exact positive right--left eigenpair $u,v$ with $G=I$,
set $w_i=u_iv_i/(v^Tu)$. The mask sensitivity coefficient is
\[
S_M:=\frac{v^TMu}{v^Tu}
=\sum_i w_i\alpha_i=\sum_j w_j\beta_j,
\qquad \sum_i w_i=1.
\]
It uses weighted means, while $r_{\rm cert}$ uses the separate
maxima of $\alpha_i$ and $\beta_j$.
The exact radius $r_*$ additionally retains the pairwise
coefficients and baseline margins.

% ----------------------------------------------------------------
% ----------------------------------------------------------------
\subsection{Trial-value errors and intrinsic level separation}
% ----------------------------------------------------------------

The updated gap combines three distinct quantities.
For $A=B+E$,
\begin{equation}
	\label{eq:gap-three-components}
	\begin{aligned}
		D_A(u,v)
		={}&\bigl[q_A(v)-\lambdaup_C(A,G)\bigr]\\
		&+\bigl[\lambdaup_C(A,G)-\lambdalow_C(A,G)\bigr]\\
		&+\bigl[\lambdalow_C(A,G)-\ell_A(u)\bigr].
	\end{aligned}
\end{equation}
All three terms are nonnegative.
The first and third are trial-value errors, whereas
the middle term is the intrinsic level separation.
Failure of a sufficient small-gap test therefore does
not establish that the retained trial values are inaccurate.
No trial gap can be smaller than the intrinsic separation.

\paragraph{An irreducible signed example.}
Let
\[
B=
\begin{pmatrix}
	2&1&0\\
	0&1&-1\\
	1&0&0
\end{pmatrix},
\qquad G=I_3.
\]
The nonzero off-diagonal entries form the directed cycle
$1\to3\to2\to1$, so $B$ is irreducible.
For every $u,v>0$,
\[
\ell_B(u)
\le 1-\frac{u_3}{u_2}<1,
\qquad
q_B(v)
\ge 2+\frac{v_3}{v_1}>2.
\]
For $0<\delta<1/2$, take
\[
u_\delta=v_\delta=(1,1,\delta)^T.
\]
Then
\[
\ell_B(u_\delta)=1-\delta,
\qquad
q_B(v_\delta)=2+\delta.
\]
The variational formulas and the limit $\delta\downarrow0$
therefore give
\[
\lambdalow_C(B,I_3)=1,
\qquad
\lambdaup_C(B,I_3)=2.
\]
Both trial-value errors equal $\delta$, while
\[
D_B(u_\delta,v_\delta)=1+2\delta.
\]
Thus arbitrarily accurate trial values coexist with
an intrinsic gap of one. Neither level is attained
by a strictly positive trial profile.

This separation persists under finite perturbations.
Let $E_{11}$ denote the coordinate matrix and set
$A_t=B+tE_{11}$, $t\ge0$.
The same argument yields
\[
\lambdalow_C(A_t,I_3)=1,
\qquad
\lambdaup_C(A_t,I_3)=2+t.
\]
At the retained profiles,
\[
\ell_{A_t}(u_\delta)=1-\delta,
\qquad
q_{A_t}(v_\delta)=2+t+\delta,
\qquad
D_{A_t}(u_\delta,v_\delta)=1+t+2\delta.
\]
Both individual errors remain equal to $\delta$,
although the gap increases with $t$.
Consequently, profile-preservation tests should distinguish
the retained gap from the individual trial-value errors,
using independently refined level bounds when available.

% ----------------------------------------------------------------
\subsection{A consequence for finite-decrease verification}
% ----------------------------------------------------------------

Let an intervention produce $B_h=B-h$, and retain
positive profiles $u,v$.
Write
\[
L_0=\ell_B(u),
\qquad
U_h=q_{B_h}(v).
\]
The baseline and post-intervention enclosures imply
\begin{equation}
	\label{eq:framework-certified-decrease}
	\lambdaup_C(B,G)-\lambdaup_C(B_h,G)
	\ge L_0-U_h.
\end{equation}
Thus a positive margin certifies strict decrease;
a nonpositive margin is inconclusive.
Neither coincident levels nor interior extremizers
are required.

The difference between the actual decrease and this
certified margin is
\[
\begin{aligned}
	&\lambdaup_C(B,G)-\lambdaup_C(B_h,G)-(L_0-U_h)\\
	&\qquad=
	\bigl[\lambdaup_C(B,G)-L_0\bigr]
	+\bigl[U_h-\lambdaup_C(B_h,G)\bigr],
\end{aligned}
\]
which is nonnegative and at most
$D_B(u,v)+D_{B_h}(u,v)$.
The first bracket includes any baseline separation
between the cone levels.

This verification applies to interventions produced by
any method, including eigenvalue or finite-difference
ranking.
For fixed $v$, the $z$-LP developed below directly
minimizes $U_h$ over the admissible polyhedral set,
thereby maximizing this particular certified margin.
Intervention construction and certificate verification
are therefore separate steps.

All statements above concern exact arithmetic.
Rigorous numerical certificates require control of
arithmetic errors.

%%%%%%%%%%%%%%%%%%%%%%%%%%%%%%%%%%%%%%%%%%%%%%%%%%%%%%%%%%%%%%%%%%%%%%%%%%%

% ================================================================
\section{Finite-budget control and numerical verification}
\label{sec:finite-budget-control}
% ================================================================

We apply the trial-profile bounds to decrease-only interventions
in the positive orthant, with
$G=\operatorname{diag}(g_1,\ldots,g_N)$, $g_i>0$.
Let $\mathcal E$ be the set of controllable entries,
where $(i,j)$ represents the edge $j\to i$.
For $\eta\ge0$ and capacities $\bar h_{ij}\ge0$, define
\begin{equation}
	\label{eq:admissible-budget-set}
	\mathcal H_\eta
	=
	\left\{
	h:
	\begin{array}{l}
		h_{ij}=0\quad\text{outside }\mathcal E,\\
		0\le h_{ij}\le\bar h_{ij}\quad\text{on }\mathcal E,\\
		\displaystyle\sum_{(i,j)\in\mathcal E}h_{ij}\le\eta
	\end{array}
	\right\}.
\end{equation}
The modified operator is $B_h=B-h$.
For entries required to remain nonnegative, assume
$B_{ij}\ge0$ and impose $\bar h_{ij}\le B_{ij}$.

% ----------------------------------------------------------------
\subsection{Direct optimization of a trial upper bound}
\label{subsec:direct-upper-bound-lp}
% ----------------------------------------------------------------

Fix a positive left trial profile $v$ and write
\[
s_j=\frac{(B^Tv)_j}{g_jv_j},
\qquad
t_j(h)=\frac{\sum_i h_{ij}v_i}{g_jv_j}.
\]
The updated upper bound is
\begin{equation}
	\label{eq:direct-finite-upper-bound}
	\lambdaup_C(B-h,G)
	\le U_v(h):=q_{B-h}(v)
	=\max_j\{s_j-t_j(h)\}.
\end{equation}
Minimizing this bound over $\mathcal H_\eta$ gives
\begin{equation}
	\label{eq:direct-upper-bound-lp}
	\begin{aligned}
		\min_{h,z}\quad &z,\\
		\text{subject to}\quad
		&(B^Tv)_j-\sum_i h_{ij}v_i
		\le z\,g_jv_j,
		\qquad j=1,\ldots,N,\\
		&h\in\mathcal H_\eta,\qquad z\in\mathbb R.
	\end{aligned}
\end{equation}
We call this the $z$-LP.
It optimizes the trial upper bound at the fixed profile $v$.

Let $L_0\le\lambdaup_C(B,G)$ be a baseline lower bound,
for example $L_0=\ell_B(u)$ with $u>0$.
Minimizing $U_v(h)$ also maximizes the associated
decrease margin $L_0-U_v(h)$.
If $h^{\rm dir}$ is an exact optimizer with value $z_v^*$,
then \eqref{eq:framework-certified-decrease} yields
\begin{equation}
	\label{eq:direct-certified-decrease}
	\lambdaup_C(B,G)-\lambdaup_C(B-h^{\rm dir},G)
	\ge L_0-z_v^*.
\end{equation}
Thus $z_v^*<L_0$ certifies strict decrease.

The uniform-margin formulation, eigenvalue and
finite-difference rankings, and their numerical comparison
are presented in
Appendix~\ref{sec:additional-control-comparisons}.
% ----------------------------------------------------------------
\subsection{Verification independent of intervention construction}
\label{subsec:finite-control-verification}
% ----------------------------------------------------------------

For any returned intervention $h$, first verify
$h\in\mathcal H_\eta$.
Then evaluate a baseline lower bound $L_0$ and a
post-intervention upper bound $q_{B-h}(v_h)$,
where $v_h>0$ may be retained or recomputed.
The margin
\[
m(h):=L_0-q_{B-h}(v_h)
\]
provides the decrease certificate
\eqref{eq:framework-certified-decrease}.
A verified positive margin certifies strict decrease;
a nonpositive margin is inconclusive.

This procedure applies to interventions constructed by
any method. Refining the post-intervention profile may
improve the certificate, but successful verification
does not require exact profile optimality or exact
solution of the intervention problem.

For numerical outputs, admissibility and trial bounds
must be checked independently of solver status and
reported objective values.
Exact rational arithmetic or outward-rounded interval
bounds provide rigorous verification.
Floating-point evaluations without arithmetic error
control are reported as numerical checks.

% ================================================================
\section{Differentiable constraints and learning applications}
\label{sec:certified-directed-learning}
% ================================================================

We use trial-profile bounds as constraints and penalties
for a differentiable operator $B_\theta$.
Throughout this section,
$\overline C=\mathbb R_+^N$ and
$G=\operatorname{diag}(g_1,\ldots,g_N)$, $g_i>0$, is fixed.
Coincidence of the cone levels is not assumed.

% ----------------------------------------------------------------
\subsection{Smooth sufficient constraints}
\label{subsec:smooth-certified-bounds}
\label{subsec:learning-cone-control}
% ----------------------------------------------------------------

Consider
\begin{equation}
	\label{eq:learning-spectral-control-problem}
	\min_\theta\mathcal L_{\rm task}(\theta)
	\quad\text{subject to}\quad
	\lambdaup_C(B_\theta,G)\le\Lambda.
\end{equation}
A positive profile $v$ supplies the sufficient constraint
$q_{B_\theta}(v)\le\Lambda$, equivalently
$B_\theta^Tv\le\Lambda Gv$.

Set
\[
r_i=\frac{(B_\theta u)_i}{g_i u_i},
\qquad
s_j=\frac{(B_\theta^Tv)_j}{g_jv_j},
\qquad u,v>0.
\]
For $\varepsilon>0$, define
\begin{align}
	\ell_{\theta,\varepsilon}(u)
	&=-\varepsilon\log\sum_i e^{-r_i/\varepsilon},
	\label{eq:smooth-lower-certificate}\\
	q_{\theta,\varepsilon}(v)
	&=\varepsilon\log\sum_j e^{s_j/\varepsilon}.
	\label{eq:smooth-upper-certificate}
\end{align}
The soft-min/max inequalities give
\begin{align}
	0&\le\ell_{B_\theta}(u)-\ell_{\theta,\varepsilon}(u)
	\le\varepsilon\log N,
	\label{eq:softmin-error}\\
	0&\le q_{\theta,\varepsilon}(v)-q_{B_\theta}(v)
	\le\varepsilon\log N.
	\label{eq:softmax-error}
\end{align}
Combining these inequalities with the trial-profile enclosure,
\begin{equation}
	\label{eq:smooth-certificate}
	\ell_{\theta,\varepsilon}(u)
	\le\lambdalow_C(B_\theta,G)
	\le\lambdaup_C(B_\theta,G)
	\le q_{\theta,\varepsilon}(v).
\end{equation}
Thus
\begin{equation}
	\label{eq:smooth-spectral-cap}
	q_{\theta,\varepsilon}(v)\le\Lambda
\end{equation}
is a differentiable sufficient condition for the cap.

Writing
\[
D_{\theta,\varepsilon}(u,v)
:=q_{\theta,\varepsilon}(v)
-\ell_{\theta,\varepsilon}(u),
\]
we obtain
\begin{equation}
	\label{eq:smooth-gap-error}
	0\le
	D_{\theta,\varepsilon}(u,v)-D_{B_\theta}(u,v)
	\le2\varepsilon\log N.
\end{equation}
The unsmoothed gap includes both trial-value errors
and intrinsic level separation, as in
\eqref{eq:gap-three-components}.

\paragraph{Sensitivity in the positive eigenpair regime.}
Let $\lambda_*$ be a simple real generalized eigenvalue
of $(B,G)$ with positive right and left vectors satisfying
\[
Bu_*=\lambda_*Gu_*,
\qquad
B^Tv_*=\lambda_*G^Tv_*.
\]
Then $\lambda_*$ equals both cone levels.
For fixed $G$, its local eigenvalue branch satisfies
\begin{equation}
	\label{eq:differential-sensitivity-general}
	D\lambda_*(B)[H]
	=\frac{v_*^THu_*}{v_*^TGu_*}.
\end{equation}
The matrix-entry sensitivity for the edge $j\to i$ is
\begin{equation}
	\label{eq:edge-sensitivity}
	g_{ij}^{\rm eig}
	=\frac{(v_*)_i(u_*)_j}{v_*^TGu_*}.
\end{equation}
Sensitivities to edge weights or model parameters follow
by the chain rule. These classical formulas support the
adaptive Perron control example below; they describe
local changes and do not themselves certify finite decrease.
The smooth trial bounds require no eigenpair assumption.

% ----------------------------------------------------------------
\subsection{Learning objective}
\label{subsec:spectral-regularization}
% ----------------------------------------------------------------

For semi-supervised node classification, take
\begin{equation}
	\label{eq:task-loss}
	\mathcal L_{\rm task}(\theta)
	=
	-\frac{1}{|\mathcal V_{\rm tr}|}
	\sum_{i\in\mathcal V_{\rm tr}}\log p_{i\,y_i}(\theta).
\end{equation}
With $[x]_+=\max\{0,x\}$, define
\begin{align}
	\mathcal L_{\rm spec}
	&=[q_{\theta,\varepsilon}(v)-\Lambda]_+^2,
	\label{eq:spectral-regularization}\\
	\mathcal L_{\rm gap}
	&=D_{\theta,\varepsilon}(u,v),
	\label{eq:gap-loss}\\
	\mathcal L_{\rm total}
	&=\mathcal L_{\rm task}
	+\beta_{\rm spec}\mathcal L_{\rm spec}
	+\beta_{\rm gap}\mathcal L_{\rm gap},
	\qquad
	\beta_{\rm spec},\beta_{\rm gap}\ge0.
	\label{eq:full-training-objective}
\end{align}
The profiles are auxiliary optimization variables.
At fixed $\theta$, minimizing the smooth gap seeks sharper
trial bounds, subject to the smoothing error.
Updating $\theta$ may also change the intrinsic separation
of the cone levels.
A finite penalty weight does not guarantee cap feasibility,
which must be checked separately.
Cone-level certification does not certify prediction accuracy
or control of transient amplification.

% ----------------------------------------------------------------
\subsection{Implementation and computational cost}
\label{subsec:mode-computation}
\label{subsec:sparse-implementation}
% ----------------------------------------------------------------

At fixed $B_\theta$, profile updates minimize
\begin{equation}
	\label{eq:mode-loss}
	\mathcal L_{\rm mode}=D_{\theta,\varepsilon}(u,v).
\end{equation}
Strict positivity and normalization are enforced by
\begin{equation}
	\label{eq:positive-mode-parametrization}
	x=
	\frac{\operatorname{softplus}(z_x)+\delta_{\rm pos}\mathbf1}
	{\|\operatorname{softplus}(z_x)+\delta_{\rm pos}\mathbf1\|_1},
	\qquad x\in\{u,v\},\qquad \delta_{\rm pos}>0.
\end{equation}
Smooth bounds are evaluated using stable log-sum-exp formulas.
The alternating workflow is:
\begin{enumerate}
	\item At fixed $\theta$, update positive profiles,
	using the previous profiles as initial values.
	Store the componentwise ratios defining
	$\ell_{B_\theta}(u)$ and $q_{B_\theta}(v)$.
	
	\item At fixed profiles, update $\theta$ using
	\eqref{eq:full-training-objective}.
	
	\item At evaluation checkpoints, compute the unsmoothed
	bounds and report their gap, cap feasibility,
	and task accuracy.
	After an operator change, the retained-profile test
	\eqref{eq:exact-updated-gap-test} supplies a sufficient
	trial-value accuracy condition.
	Failure is inconclusive and may motivate profile refinement
	or a sharper error assessment.
\end{enumerate}
For a control intervention, construction and verification
follow Section~\ref{sec:finite-budget-control}.
Rigorous certificates require arithmetic error control.

LP feasibility and bisection provide an alternative means
of computing profiles; eigensolvers are also available
in the positive spectral regime.
Update schedules and hyperparameters are specified
in the experimental protocols.

The SBM and Cora experiments follow their respective
profile-update procedures.
They do not test a training strategy in which the
retained-profile criterion triggers recomputation.
Subsection~\ref{subsec:profile-preservation-experiment}
examines this criterion separately under prescribed
finite perturbations.

\paragraph{Computational cost.}
Let $m_B$ be the number of nonzero entries in the sparse
part of $B$.
With an additional rank-$r$ term in factored form,
evaluating the forward and transpose products,
trial bounds, and smooth cone terms costs
\begin{equation}
	\label{eq:certificate-complexity}
	O(m_B+Nr+N).
\end{equation}
At fixed $B$, $k$ profile-gradient steps have cost
$O(k(m_B+Nr+N))$.
Model-specific forward and backward computations are additional.
Dense unstructured operators require $O(N^2)$ per product.
Structured updates of retained-profile bounds have the costs
given in Remark~\ref{rem:structured-perturbation-cost}.
Profile computation by LP or eigensolvers and rigorous
arithmetic verification have separate costs.

\paragraph{Stopping criteria.}
The unsmoothed condition
\begin{equation}
	\label{eq:stopping-gap}
	D_{B_\theta}(u,v)
	=q_{B_\theta}(v)-\ell_{B_\theta}(u)
	\le\eta_{\rm gap}
\end{equation}
is a sufficient trial-value accuracy test.
It is unattainable when the intrinsic level separation
exceeds $\eta_{\rm gap}$; profile parametrization
and optimization may impose further limitations.
An iteration limit and a stagnation test are therefore
also used.
A small floating-point gap is a numerical stopping indicator;
rigorous certification requires arithmetic error control,
and numerical closeness does not prove exact coincidence
of the cone levels.

%%%%%%%%%%%%%%%%%%%%%%%%%%%%%%%%%%%%%%%%%%%%%%%%%%%%%%%%%%%%%%%%%%%%%%%%%%%%%%%%%%%%%%%%%%%%%%%%%%%%

% ================================================================
\section{Numerical experiments}
\label{sec:experiments}
% ================================================================

We first examine preservation of trial-profile accuracy under
finite perturbations, including comparison with cold- and
warm-start recomputation.
We then give one example of verified finite-budget control
and summarize two applications in the positive spectral regime.
Additional intervention comparisons are provided in the
supplementary material.
All experiments use $G=I$ and the positive orthant.
Floating-point level estimates are distinguished from
rationally verified trial bounds.

% ----------------------------------------------------------------
\subsection{Test operators and numerical level estimation}
\label{subsec:interior-boundary-control-transition}
% ----------------------------------------------------------------

The reference graph is a strongly connected directed stochastic
block model with $N=100$, two equal classes, and connection
probabilities
\[
P=
\begin{pmatrix}
	0.12&0.08\\
	0.02&0.12
\end{pmatrix},
\]
where rows index source classes.
Seed $7$ gives $840$ edges.
The nonnegative reference operator is
\begin{equation}
	\label{eq:num-reference-operator}
	B=I+D_{\rm in}^{-1/2}(A\odot W)D_{\rm out}^{-1/2},
	\qquad W_{ij}\sim\operatorname{Unif}[0.05,1].
\end{equation}
Five signed operators are defined by
\begin{equation}
	\label{eq:signed-perturbation-family}
	M_k=B+\rho_kH_k,
	\qquad
	(\rho_1,\ldots,\rho_5)=(0.20,0.20,0.35,0.20,0.40),
\end{equation}
where $H_k$ is a zero-diagonal Gaussian matrix generated
with seed $200000+k$ and normalized by its spectral norm.
These operators share one graph.

Numerical level estimates are obtained by LP feasibility
and bisection using
\[
Mu\ge\lambda u,\qquad M^Tv\le\lambda v,
\qquad
u,v\in\Delta_\delta,
\qquad
\Delta_\delta=\{x:\mathbf1^Tx=1,\ x_i\ge\delta\}.
\]
The exact truncated problems give a lower bound for the lower
cone level and an upper bound for the upper cone level.
Floating-point feasibility alone does not make the resulting
estimates rigorous.
The profile-recomputation settings are specified below;
the intervention-comparison settings are given in the
supplementary material.

At the selected amplitudes, the continued left eigenvector
changes sign in Direction~1, the right eigenvector in
Direction~2, and both in Directions~3--5.
For Direction~4, positivity is lost numerically near
$\rho=0.122442$.
At $\rho=0.20$,
\[
\lambda_{\rm eig}=1.485009261,\qquad
\ell^{\rm num}=1.473978514,\qquad
U^{\rm num}=1.473978801.
\]
The separation from the continued eigenvalue is about
$1.10\times10^{-2}$.
This does not establish exact coincidence of the cone levels
or exclude other eigenvalues near their numerical interval.

% ----------------------------------------------------------------
\subsection{Preservation and recomputation of trial profiles}
\label{subsec:profile-preservation-experiment}
% ----------------------------------------------------------------

For each $A_0\in\{B,M_1,\ldots,M_5\}$, we retain a positive
pair $u,v$ and the component ratios defining
$\ell_{A_0}(u)$ and $q_{A_0}(v)$.
The reference profiles for $B$ are computed from its Perron pair;
the signed cases use LP-generated profiles.
The saved matrices and profiles are used for Direction~4.
The other four signed matrices are reconstructed from the
archived generator and the saved reference operator;
the resulting inputs are included in the reproducibility data.

We consider two fixed directions $H$: a localized perturbation
of ten off-diagonal entries and a rank-two perturbation
$H=XY^T$.
Both have unit Frobenius norm and are generated using seed
$20261002$, independently of their subsequent performance.
The perturbed operator is $A_0+tH$, $t\ge0$.
These tests permit arbitrary signs and are not restricted
to the decrease-only interventions considered below.

\paragraph{Accuracy tests and amplitudes.}
With $\varepsilon=10^{-4}$, we compare
\begin{equation}
	\label{eq:num-profile-gap-comparison}
	\begin{aligned}
		D_{A_0+tH}(u,v)
		&\le D_{A_0}(u,v)
		+t\bigl[b_+(H;v)-a_-(H;u)\bigr]\\
		&\le D_{A_0}(u,v)
		+t\bigl[\alpha_{|H|}(u)+\beta_{|H|}(v)\bigr].
	\end{aligned}
\end{equation}
The last expression gives the uniform radius
$r_{\rm cert}$ from
Corollary~\ref{cor:uniform-profile-radius}.
For each operator--direction pair, we test
\[
t/r_{\rm cert}\in\{0.25,0.5,1,2,4\}.
\]
We also determine numerically the retained-gap threshold
\[
r_{\rm ret}
:=\sup\{t\ge0:D_{A_0+tH}(u,v)\le\varepsilon\}
\]
and test $t/r_{\rm ret}\in\{0.5,1,2\}$.
This gives $96$ cases.
The sampled amplitudes were generated using floating-point
radius estimates. Table~\ref{tab:optimal-profile-radii}
reports their subsequent exact rational reevaluation,
together with the optimal uniform radius $r_*$.
This reevaluation does not change the stored perturbed
operators used for profile recomputation.

The threshold $r_{\rm ret}$ is computed from the affine
component ratios: the gap condition is equivalent to requiring
every updated column ratio minus every updated row ratio
to be at most $\varepsilon$.
It concerns the selected direction, whereas $r_{\rm cert}$
guarantees accuracy throughout the corresponding mask class.

\paragraph{Sharpness and preservation.}
For the positive reference pair, the directional bound
in \eqref{eq:num-profile-gap-comparison} is nearly exact,
as predicted by Proposition~\ref{cor:finite-interior-pair}.
For signed operators, retaining the baseline component
ratios can yield a substantially sharper test.

To distinguish conservatism of the sufficient radius
from directional variation in profile robustness,
we additionally compute the optimal uniform radius
$r_*$ from \eqref{eq:optimal-uniform-profile-radius},
using the same stored profiles and the mask $M=|H|$.

\begin{table}[t]
	\centering
	\small
	\setlength{\tabcolsep}{3pt}
	\caption{Sufficient and optimal uniform trial-gap radii, and
		thresholds along the selected directions, for $G=I$ and
		$\varepsilon=10^{-4}$. The mask is $M=|H|$ in each case.
		$r_*$ is optimal over all independent signed entrywise
		perturbations satisfying $|E|\le rM$; the rank-two designation
		refers to the selected direction, not to this full mask class.
		Displayed values are rounded from exact rational calculations
		for the stored numerical inputs.}
	\label{tab:optimal-profile-radii}
	\begin{tabular}{llrrrrr}
		\toprule
		Operator & Direction & $r_{\rm cert}$ & $r_*$ & $r_{\rm ret}$
		& $r_*/r_{\rm cert}$ & $r_{\rm ret}/r_*$ \\
		\midrule
		$B$ & Local & $6.545\times10^{-5}$ & $6.545\times10^{-5}$ & $9.790\times10^{-5}$ & $1.000$ & $1.496$ \\
		$B$ & Rank two & $1.030\times10^{-5}$ & $1.061\times10^{-5}$ & $1.381\times10^{-4}$ & $1.031$ & $13.017$ \\
		$M_1$ & Local & $4.062\times10^{-5}$ & $4.062\times10^{-5}$ & $5.204\times10^{-5}$ & $1.000$ & $1.281$ \\
		$M_1$ & Rank two & $1.570\times10^{-11}$ & $1.036\times10^{-5}$ & $1.299\times10^{-4}$ & $6.600\times10^{5}$ & $12.536$ \\
		$M_2$ & Local & $3.528\times10^{-5}$ & $3.528\times10^{-5}$ & $8.461\times10^{-5}$ & $1.000$ & $2.398$ \\
		$M_2$ & Rank two & $2.312\times10^{-10}$ & $2.551\times10^{-10}$ & $2.687\times10^{-9}$ & $1.103$ & $10.530$ \\
		$M_3$ & Local & $1.187\times10^{-9}$ & $1.094\times10^{-5}$ & $6.016\times10^{-5}$ & $9.215\times10^{3}$ & $5.500$ \\
		$M_3$ & Rank two & $3.209\times10^{-10}$ & $3.963\times10^{-10}$ & $1.933\times10^{-8}$ & $1.235$ & $48.771$ \\
		$M_4$ & Local & $9.085\times10^{-10}$ & $1.047\times10^{-9}$ & $2.490\times10^{-5}$ & $1.152$ & $2.379\times10^{4}$ \\
		$M_4$ & Rank two & $2.254\times10^{-10}$ & $5.487\times10^{-10}$ & $1.067\times10^{-7}$ & $2.435$ & $194.456$ \\
		$M_5$ & Local & $2.333\times10^{-5}$ & $2.553\times10^{-5}$ & $3.523\times10^{-5}$ & $1.094$ & $1.380$ \\
		$M_5$ & Rank two & $6.375\times10^{-10}$ & $8.633\times10^{-10}$ & $6.936\times10^{-8}$ & $1.354$ & $80.344$ \\
		\bottomrule
	\end{tabular}
\end{table}

\begin{table}[htbp]
	\centering
	\small
	\setlength{\tabcolsep}{4pt}
	\caption{Component ranges of the retained baseline profiles,
		normalized by $\mathbf1^Tu=\mathbf1^Tv=1$.
		The balance indicator $\zeta$ is defined in the text;
		smaller values indicate more uneven components.
		Each pair is used for both perturbation directions in
		Table~\ref{tab:optimal-profile-radii}.
		Values are rounded from the stored numerical profiles.}
	\label{tab:baseline-profile-components}
	\begin{tabular}{lrrrrr}
		\toprule
		Operator & $\min_i u_i$ & $\max_i u_i$
		& $\min_j v_j$ & $\max_j v_j$ & $\zeta(u,v)$ \\
		\midrule
		$B$ & $3.052\times10^{-3}$ & $2.155\times10^{-2}$ & $2.160\times10^{-3}$ & $1.930\times10^{-2}$ & $1.119\times10^{-1}$ \\
		$M_1$ & $2.207\times10^{-6}$ & $2.125\times10^{-2}$ & $1.291\times10^{-9}$ & $2.257\times10^{-2}$ & $5.721\times10^{-8}$ \\
		$M_2$ & $3.983\times10^{-8}$ & $2.063\times10^{-2}$ & $1.438\times10^{-7}$ & $1.042\times10^{-1}$ & $1.381\times10^{-6}$ \\
		$M_3$ & $5.284\times10^{-8}$ & $2.370\times10^{-2}$ & $2.000\times10^{-7}$ & $6.914\times10^{-2}$ & $2.230\times10^{-6}$ \\
		$M_4$ & $3.121\times10^{-8}$ & $2.499\times10^{-2}$ & $1.740\times10^{-7}$ & $1.283\times10^{-1}$ & $1.249\times10^{-6}$ \\
		$M_5$ & $1.331\times10^{-7}$ & $2.427\times10^{-2}$ & $1.971\times10^{-7}$ & $3.331\times10^{-2}$ & $5.483\times10^{-6}$ \\
		\bottomrule
	\end{tabular}
\end{table}

\paragraph{Profile component ranges and robustness.}
Table~\ref{tab:baseline-profile-components} reports the component
ranges of the fixed profiles used in the radius calculations.
For a scale-invariant measure of component balance, define
\[
\zeta(u,v):=
\min\left\{
\frac{\min_i u_i}{\max_i u_i},
\frac{\min_j v_j}{\max_j v_j}
\right\}.
\]
Small components can amplify the perturbation ratios, but
$r_*$ depends jointly on the coefficients $\kappa_{ij}$,
the baseline margins $\varepsilon-d_{ij}$, and the mask.
For example, $M_1$ has the smallest component of any signed
baseline profile, yet both optimal radii are of order $10^{-5}$.
For $M_2$, the same profiles give radii of order $10^{-5}$
and $10^{-10}$ under the localized and rank-two masks,
respectively.

Across the six operators, the descriptive Spearman correlation
between $\zeta$ and $r_*$ is approximately $0.257$ for each
mask type; restricting the calculation to the five signed
operators gives $-0.300$ for each type.
These descriptive correlations, based on only six operators,
do not establish a general relationship between component
balance and the optimal radius.
The contrasting examples show that component minima alone
do not predict robustness: baseline margins and the perturbation
mask must also be considered.
The effect of the profile-generation procedure and its
positivity floor requires a separate sensitivity study.
The floor $10^{-10}$ specified for profile recomputation does
not describe a common generation procedure for all retained
baseline profiles.

\paragraph{Uniform radii and directional thresholds.}

Table~\ref{tab:optimal-profile-radii} distinguishes the loss
in the sufficient radius from the difference between a
selected direction and the worst case over its mask class.
For Direction~4, $r_*/r_{\rm cert}$ is approximately $1.15$
for the localized direction and $2.43$ for the rank-two
direction, whereas $r_{\rm ret}/r_*$ is approximately
$2.38\times10^4$ and $194$, respectively.
Thus, for these fixed profiles, the small optimal uniform
radii reflect the growth of the retained-profile gap
under adverse perturbations allowed by the mask.
The much larger directional thresholds do not contradict
this worst-case limitation.
For $M_1$ with the rank-two direction, however, the exact
uniform radius improves $r_{\rm cert}$ by approximately
$6.60\times10^5$; for $M_3$ with the localized direction,
the improvement is approximately $9.22\times10^3$.

All twelve optimal radii were evaluated using exact rational
arithmetic for the stored matrices, profiles, and materialized
directions. The direction generated from rank-two factors is stored as
$H=\operatorname{fl}(XY^T)$.
The rational calculations use this materialized matrix;
they do not require it to have exact rank two. The resulting radii concern the exact
mathematical perturbations of these stored inputs.
For every operator--direction pair, a worst-case direction
$K$, $|K|\le M$, constructed from the active index pair satisfies
\[
D_{A_0+r_*K}(u,v)=\varepsilon,
\qquad
D_{A_0+(101/100)r_*K}(u,v)>\varepsilon
\]
in exact rational arithmetic.
These checks do not apply the extra rounding operation
$\operatorname{fl}(A_0+rK)$ used for the stored perturbed
operators in the profile-recomputation experiment.
The optimality statements concern the trial gap, not the
individual trial-value errors.

Rational verification of the stored perturbed operators
confirms the retained-gap tolerance in all twelve cases
at $t=r_{\rm cert}$.
At $t=2r_{\rm cert}$, nine cases satisfy it, and at
$t=4r_{\rm cert}$, eight do.
These counts concern the selected directions and do not
extend the uniform guarantee beyond its proven radius.

Figure~\ref{fig:profile-preservation-bounds} complements
Table~\ref{tab:optimal-profile-radii} by showing the bounds
and retained gaps along the selected perturbation directions.
The mask curve represents the sufficient bound underlying
$r_{\rm cert}$; it is not the exact worst-case gap from
\eqref{eq:exact-mask-worst-gap}.

\begin{figure}[htbp]
	\centering
	\includegraphics[width=\linewidth]{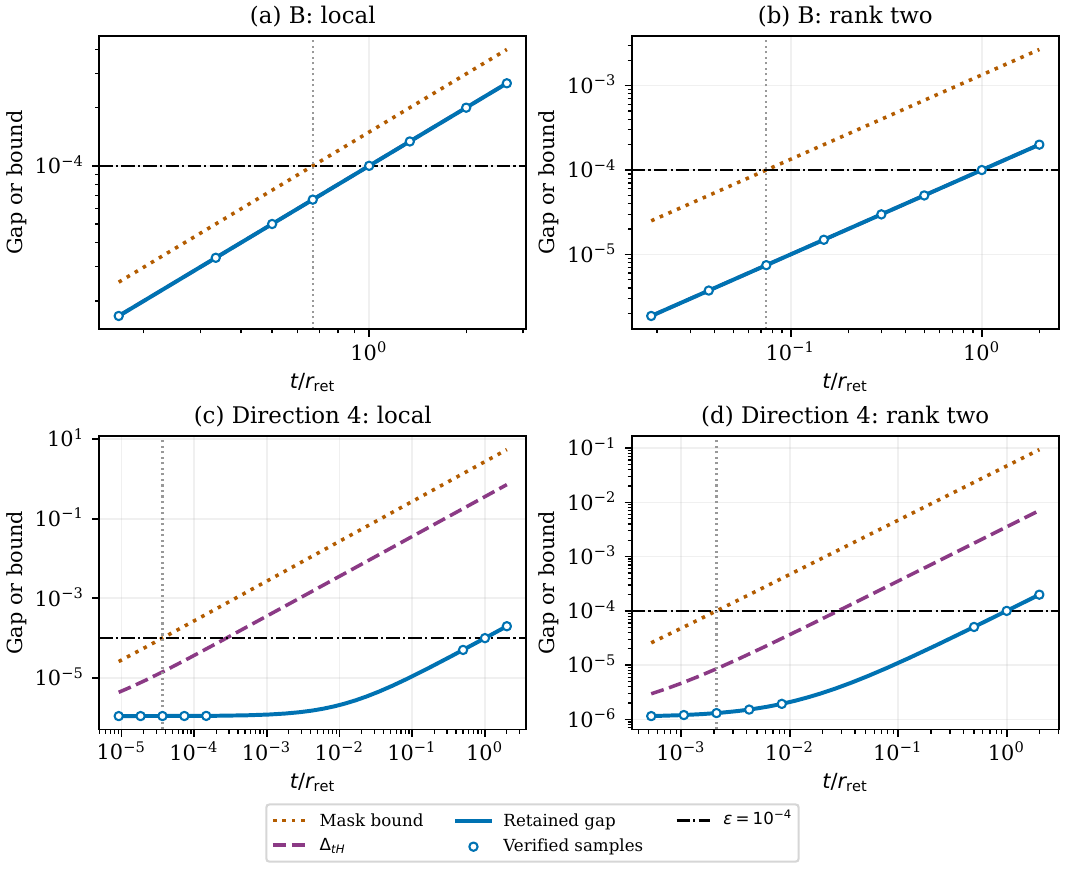}
	\caption{Retained-profile gaps and perturbation bounds for the nonnegative
		reference operator $B$ and signed Direction~4 under local and rank-two
		perturbations. The horizontal coordinate is $t/r_{\rm ret}$, where
		$r_{\rm ret}$ is the numerically determined threshold for the retained
		gap along the selected direction. Curves show the mask bound
		$D_{A_0}+t(\alpha_{|H|}+\beta_{|H|})$, the directional bound
		$\Delta_{tH}$, and the retained gap computed from the affine component
		ratios. Circles show rationally verified gaps for the stored rounded
		operators $\widetilde A=\operatorname{fl}(A_0+tH)$.
		The horizontal line is the tolerance $\varepsilon=10^{-4}$;
		the vertical dotted line marks $r_{\rm cert}/r_{\rm ret}$.
		The directional bound is nearly indistinguishable from the retained
		gap in the two panels for $B$. Axis ranges differ between panels.}
	\label{fig:profile-preservation-bounds}
\end{figure}

\paragraph{Recomputation and genuine warm starts.}
For every perturbed operator, we recompute positive profiles
by LP bisection with tolerance $10^{-7}$ and positivity
floor $10^{-10}$, followed by recovery that maximizes
the smallest profile component at a slightly relaxed
feasible level.
Returned profiles are checked through their component ratios.

The cold variant starts without baseline LP bases.
The warm variant transfers the bases computed for $A_0$
through the HiGHS basis interface.
Both variants use the same initial interval and accuracy
criteria, and both reuse bases within each bisection.
Four timed repetitions alternate their execution order.
The timings include model construction, bisection,
profile recovery, and ratio checks, but exclude initial
profile and basis preparation, operator formation,
and exact rational verification.

\begin{table}[htbp]
	\centering
	\caption{Profile recomputation on $16$ cases per operator.
		Times are medians; speedup is the median of paired
		cold/warm time ratios.}
	\label{tab:profile-recomputation-cost}
	\small
	\begin{tabular}{lrrr}
		\toprule
		Operator & Cold (ms) & Warm (ms) & Paired speedup\\
		\midrule
		$B$         & 51.44 & 41.20 & 1.24\\
		$M_1$ & 93.56 & 76.94 & 1.20\\
		$M_2$ & 91.29 & 75.31 & 1.21\\
		$M_3$ & 69.03 & 47.07 & 1.42\\
		$M_4$ & 58.78 & 47.19 & 1.23\\
		$M_5$ & 93.83 & 73.95 & 1.28\\
		\bottomrule
	\end{tabular}
\end{table}

Warm-start median times are smaller in all $96$ cases.
Evaluation of the retained-profile gap and directional bound
takes approximately $14$--$15$ microseconds in median,
using the localized or factorized products with $H$.
It is measured in five batches of $300$ evaluations.

This inexpensive computation tests a sufficient accuracy
criterion; recomputation produces much narrower bounds.
The timings compare the cost of checking the sufficient
retained-gap criterion with that of recomputing profiles
to a tighter tolerance. Failure of the criterion does
not by itself establish that recomputation is necessary.
The two procedures therefore do not target identical
output accuracy.

\paragraph{Verified enclosures and errors of retained profiles.}
All $192$ recovered pairs pass exact rational verification,
with trial gaps below $5.0\times10^{-7}$.
The certificates concern
$\widetilde A=\operatorname{fl}(A_0+tH)$,
the stored binary64 matrix passed to the solver.
They do not assert exact coincidence of the cone levels
or automatically certify the unrounded matrix $A_0+tH$.

For the retained pair, write
$L_0=\ell_{\widetilde A}(u)$ and
$U_0=q_{\widetilde A}(v)$.
Let $L_1$ be the larger of the two recomputed lower bounds
and $U_1$ the smaller of the recomputed upper bounds.
Then
\begin{equation}
	\label{eq:num-retained-error-enclosures}
	\begin{aligned}
		\max\{0,L_1-L_0\}
		&\le \lambdalow_C(\widetilde A,I)-L_0
		\le \min\{U_1,U_0\}-L_0,\\
		\max\{0,U_0-U_1\}
		&\le U_0-\lambdaup_C(\widetilde A,I)
		\le U_0-\max\{L_1,L_0\}.
	\end{aligned}
\end{equation}
These rationally evaluated intervals bound the individual
errors of the retained profiles; their widths do not exceed
$U_1-L_1$.
The intrinsic level separation is also bounded by $U_1-L_1$.
Thus the recomputed enclosures assess retained-profile errors
without identifying the two levels or treating numerical
LP optimality as a certificate.

Figure~\ref{fig:retained-profile-errors} reports the individual
trial-value errors bounded by
\eqref{eq:num-retained-error-enclosures}.
The error bars display certified uncertainty in the individual
trial-value errors; their construction does not assume
coincidence of the cone levels.

\begin{figure}[htbp]
	\centering
	\includegraphics[width=\linewidth]{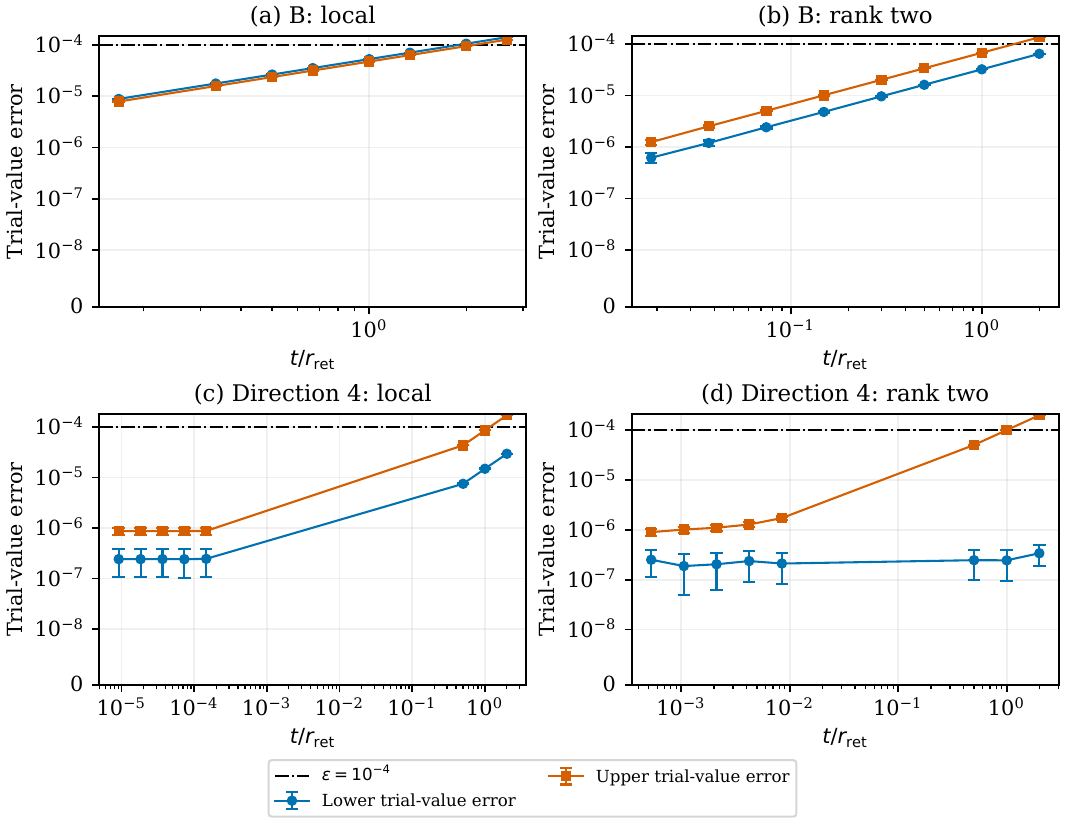}
	\caption{Verified intervals for the lower and upper trial-value errors
		of retained profiles for $B$ and Direction~4.
		Markers are interval midpoints; error bars display the lower and
		upper bounds in \eqref{eq:num-retained-error-enclosures}, obtained
		using the independently recomputed profiles.
		These are deterministic enclosures, not statistical confidence
		intervals. Endpoints are displayed as floating-point approximations
		of the exact rational bounds for the stored operators.
		Lines connect sampled amplitudes only as visual guides.
		The horizontal line is $\varepsilon=10^{-4}$.
		The horizontal axis is logarithmic; the vertical axis is logarithmic
		above $10^{-8}$ and linear near zero. Axis ranges differ between panels.}
	\label{fig:retained-profile-errors}
\end{figure}

% ----------------------------------------------------------------
\subsection{Verified finite-budget control}
\label{subsec:numerical-verification-cost}
% ----------------------------------------------------------------

We illustrate the construction and verification from
Section~\ref{sec:finite-budget-control} using $M_4$
(Direction~4) and budget $\eta=0.04$.
Only original graph entries are controllable, with feasible set
\begin{equation}
	\label{eq:num-budget-set}
	\mathcal H_{0.04}
	=\left\{h:
	\begin{array}{l}
		h_{ij}=0\quad\text{outside }\mathcal E,\\
		0\le h_{ij}\le\min\{0.002,\max\{(M_4)_{ij},0\}\}
		\quad\text{on }\mathcal E,\\
		\displaystyle\sum_{(i,j)\in\mathcal E}h_{ij}\le0.04
	\end{array}
	\right\}.
\end{equation}
The $z$-LP uses a fixed positive left profile recovered
by maximizing its smallest component subject to
$M_4^Tv\le\lambda_{\rm test}v$ and $\mathbf1^Tv=1$.
The constraints are divided by $v_j$ before solving.

Positive baseline profiles supply
\begin{equation}
	\label{eq:boundary-corrected-decrease}
	L_0=\ell_{M_4}(u),\qquad
	U_v(h)=q_{M_4-h}(v),\qquad
	m(h)=L_0-U_v(h).
\end{equation}
Exact rational evaluation verifies the stored intervention's
admissibility and gives the fixed-profile certificate
\[
\lambdaup_C(M_4,I)-\lambdaup_C(M_4-h,I)
\ge m(h)>9.266\times10^{-4}.
\]

We subsequently apply a common left-profile refinement
procedure to all $120$ interventions from the four-method
comparison on the five signed matrices.
The interventions remain unchanged.
For each $h$, we recover a positive profile $v_h$ and use
the better of the retained and refined upper bounds:
\[
U_{\rm best}(h)
=\min\{q_{M-h}(v),q_{M-h}(v_h)\},
\qquad
m_{\rm best}(h)=L_0-U_{\rm best}(h).
\]
Independent exact rational verification confirms admissibility
and positive decrease margins for all $120$ interventions.

For Direction~4 at $\eta=0.04$, the resulting verified margins
are approximately
\[
\begin{array}{c|cccc}
	\text{Method}
	&z\text{-LP}&\gamma\text{-LP}&\text{Eigenvalue}&\text{FD}\\
	\hline
	10^3m_{\rm best}
	&0.927394&0.722710&1.092216&2.345754
\end{array}
\]
These are rounded displays of exact rational lower bounds
on the decrease.
The certificates concern the exact differences of stored
binary64 matrices and interventions, without rounding
$M-h$ back to binary64.
Appendix~\ref{sec:additional-control-comparisons} reports
the full comparison, refinement protocol, and separate
construction and verification timings.

% ----------------------------------------------------------------
\subsection{Applications in the positive spectral regime}
\label{subsec:positive-regime-applications}
% ----------------------------------------------------------------

\paragraph{Directed SBM learning.}
\label{subsec:certified-directed-sbm}
We use $50$ independent graphs with the same connection
probabilities, eight-dimensional Gaussian features,
and a $20/20/60$ train/validation/test split.
The operator has form \eqref{eq:num-reference-operator}
with trainable feature-dependent weights in $[0.05,1]$
and target cap $\Lambda=1.55$.
Training and numerical cap checks use trial profiles
without an eigensolver.

Post-hoc spectral computations confirm the cap numerically
in all $50$ runs; $44$ require intervention.
The mean Perron root changes from
$1.575696\pm0.026550$ to $1.546240\pm0.008854$,
and test accuracy from $0.8710\pm0.0521$ to
$0.8717\pm0.0536$ (sample standard deviations).
The paired accuracy change has $95\%$ confidence interval
$[-0.0036,0.0049]$.
The reported cap checks use floating-point arithmetic.

\paragraph{Adaptive Perron control on Cora.}
\label{subsec:cora-adaptive-sensitivity}
The directed Cora experiment~\cite{SenEtAl2008}
uses $2708$ vertices, $5429$ edges, $1433$ features,
seven classes, and one stratified split of
$1353/402/953$ vertices.
The learned operator is
\[
B_\theta=I+\beta T_W,
\qquad
T_W=(1-\alpha)D_{\rm in}^{-1/2}
(A\odot W)D_{\rm out}^{-1/2}
+\frac{\alpha}{N}\mathbf1\mathbf1^T,
\qquad \alpha=0.01.
\]
Post-training scores follow from
\eqref{eq:edge-sensitivity} by the chain rule.
Adaptive control recomputes the Perron pair at each step
using matrix-free ARPACK, with positive power iteration
as fallback; fixed control retains the initial ranking.

At a cumulative weight-reduction budget of $0.5\%$
of the initial total weight, the numerical Perron level
decreases from $1.786681$ to $1.402730$ under adaptive
control, compared with $1.507341$ under fixed ranking.

The large reduction reflects strongly localized Perron modes,
rather than unusually large edge weights: initially, two
vertices carry $96.8\%$ of the mass of each $\ell^1$-normalized
Perron vector, while the edge weights have mean $1.0014$
and standard deviation $0.0080$.
The two reciprocal edges joining these vertices have
sensitivities approximately $0.392$ and $0.389$.
Their weights decrease from $1.004$ and $1.010$ to
$0.500$ and $0.521$, respectively.
Thus a small global weight budget permits substantial
local changes on spectrally influential edges.

Test accuracy remains $0.8153$ at all seven positive-budget
adaptive checkpoints without retraining.
Their floating-point trial gaps range from
$2.51\times10^{-9}$ to $2.44\times10^{-8}$.
These results concern the single trained model and split
specified above.
Full learning protocols and hyperparameters are provided
in the supplementary material.

% ================================================================
\section{Conclusion}
\label{sec:discussion-conclusion}
% ================================================================

We formulated the reuse of trial profiles as a finite perturbation
certification problem for nonsymmetric matrix pencils.
For fixed positive profiles and positive diagonal $G$, we
determined the exact worst-case trial gap over an independent
entrywise perturbation class and the largest uniform radius
$r_*$ meeting a prescribed gap tolerance.
The direction-dependent bound $\Delta_E$ provides an inexpensive
sufficient test for individual perturbations.
The analysis treats both cone levels and separates trial-value
errors from intrinsic level separation, as illustrated by
an irreducible signed example with exactly distinct levels.

Experiments on a nonnegative reference operator and five signed
matrices distinguish conservatism of the simpler sufficient
radius from directional variation in profile robustness.
The optimal radius substantially improves the sufficient
radius in several cases and identifies genuine fixed-profile
limitations in others.
Exact rational checks verify the radii, worst-case constructions,
and trial enclosures for the stored numerical inputs.
Cold- and warm-start recomputation supplies refined enclosures,
with smaller median recomputation times under warm starts
in all tested cases.

Finite-budget control applies the trial bounds to intervention
construction and verification.
The $z$-LP optimizes a fixed-profile upper bound, while common
post-intervention profile refinement yields positive exact
rational decrease certificates for all $120$ tested interventions.
The comparison thereby assesses intervention construction
and subsequent certification separately.
Smooth bounds additionally support differentiable constraints,
with SBM and Cora illustrating operator control in the
positive spectral regime.

The uniform radius is optimal for fixed profiles and the
specified entrywise perturbation class; its practical size
depends on profile selection and the mask.
The signed-matrix experiments use a common synthetic graph,
and the Cora application uses one trained model and split.
Profile accuracy concerns trial values, and adaptive reuse
during training remains to be evaluated.

Further work should develop profile selection that balances
initial accuracy with perturbation robustness, sharper guarantees
for restricted perturbation classes, and adaptive recomputation
driven by retained-profile bounds.

% ================================================================
\appendix
\section{Additional finite-budget comparisons}
\label{sec:additional-control-comparisons}
% ================================================================

This appendix specifies the four intervention methods, their
common admissible set, and the retained- and refined-profile
verification protocols. It reports observed reductions,
exact decrease margins, and computational timings.

% ----------------------------------------------------------------
\subsection{Intervention methods}
\label{subsec:additional-control-methods}
% ----------------------------------------------------------------

The $z$-LP is defined in \eqref{eq:direct-upper-bound-lp}.
For comparison, fix the same positive left profile $v$ and set
\[
s_j=\frac{(B^Tv)_j}{g_jv_j},\qquad
t_j(h)=\frac{\sum_i h_{ij}v_i}{g_jv_j},\qquad
\gamma_v(h)=\min_j t_j(h).
\]
The uniform-margin formulation is
\begin{equation}
	\label{eq:certified-finite-budget-control}
	\begin{aligned}
		\max_{h,\gamma}\quad &\gamma,\\
		\text{subject to}\quad
		&\sum_i h_{ij}v_i\ge\gamma g_jv_j,
		\qquad j=1,\ldots,N,\\
		&h\in\mathcal H_\eta,\qquad\gamma\ge0.
	\end{aligned}
\end{equation}
It optimizes the relaxation
$q_{B-h}(v)\le q_B(v)-\gamma_v(h)$.
If $h^{\rm unif}$ is an optimizer with value $\gamma_v^*$,
and $z_v^*$ is the optimal $z$-LP value, then
\begin{equation}
	\label{eq:direct-uniform-comparison}
	z_v^*\le q_{B-h^{\rm unif}}(v)
	\le q_B(v)-\gamma_v^*.
\end{equation}
With $d_j=q_B(v)-s_j\ge0$,
\[
q_B(v)-q_{B-h}(v)=\min_j\{d_j+t_j(h)\}.
\]
The $z$-LP uses these baseline slacks, whereas the $\gamma$-LP
requires a common decrement across all columns.
The objectives are equivalent when all $s_j$ coincide.
An unmodifiable column forces $\gamma_v^*=0$, even when
the direct upper bound can decrease.

\paragraph{Eigenvalue ranking.}

Eigenvector-based edge ranking is an established approach
to spectral network intervention
\cite{VanMieghemEtAl2011,ChenEtAl2016}.

For a simple real generalized eigenvalue of a regular pencil,
let $u_*,v_*$ be real right and left eigenvectors with
$v_*^TGu_*\ne0$. At fixed $G$, the matrix-entry scores are
\[
g_{ij}^{\rm eig}
=\frac{(v_*)_i(u_*)_j}{v_*^TGu_*}.
\]
Under a decrease-only intervention, the first-order eigenvalue
change is
\[
-\sum_{(i,j)\in\mathcal E}g_{ij}^{\rm eig}h_{ij}.
\]
The intervention fills capacities in decreasing order of positive
scores, minimizing this linearization under the budget.

\paragraph{Finite-difference ranking.}
FD ranks entries by the decrease of a numerical midpoint of the
two cone-level estimates under a coordinate probe.
The score and diagnostic step are specified below.

\subsection{Four-method finite-budget comparison}
\label{subsubsec:boundary-cone-control}
% ----------------------------------------------------------------

Here $G=I$. Numerical level estimation uses
$\delta=10^{-8}$, bisection tolerance $2\times10^{-7}$,
and HiGHS primal and dual tolerances $10^{-10}$.
Only original graph entries are controllable.
For each signed matrix $M$, the common feasible set is
\begin{equation}
	\label{eq:supp-num-budget-set}
	\mathcal H_\eta=
	\left\{
	h:
	\begin{array}{l}
		h_{ij}=0\quad\text{outside }\mathcal E,\\
		0\le h_{ij}\le\min\{0.002,\max\{M_{ij},0\}\}
		\quad\text{on }\mathcal E,\\
		\displaystyle\sum_{(i,j)\in\mathcal E}h_{ij}\le\eta
	\end{array}
	\right\}.
\end{equation}
We compare the $z$-LP, the $\gamma$-LP, continued-eigenvalue
ranking, and FD ranking at
\[
\eta\in\{0.002,0.010,0.020,0.040,0.100,0.200\}.
\]
Both LPs use the same positive left trial profile, recovered
by maximizing its smallest component subject to
$M^Tv\le\lambda_{\rm test}v$ and $\mathbf1^Tv=1$.
The $z$-LP constraints are divided by $v_j$ before solving,
and the returned interventions are checked independently.

For the continued real eigenvalue branch, eigenvalue ranking
uses $v_i^{\rm eig}u_j^{\rm eig}$,
with $(v^{\rm eig})^Tu^{\rm eig}=1$.
The implemented FD score is
\begin{equation}
	\label{eq:num-fd-score}
	\begin{aligned}
		g_{ij}^{\rm FD}
		&=
		\frac{
			\widehat{\lambda}_C(M)
			-\widehat{\lambda}_C(M-\xi E_{ij})
		}{\xi},
		\qquad \xi=0.002,\\
		\widehat{\lambda}_C(M)
		&=
		\frac12\bigl(\ell^{\rm num}(M)+U^{\rm num}(M)\bigr).
	\end{aligned}
\end{equation}
The midpoint is the numerical ranking target.
The diagnostic step is identical for every tested edge and is
not truncated to its capacity; the probe may therefore exceed
the admissible reduction or make the probed entry negative.
The final interventions satisfy the capacities and budget
in \eqref{eq:supp-num-budget-set}.

Both ranking methods construct the final intervention
by filling edge capacities in decreasing order of positive
scores, subject to the remaining budget.
The returned interventions are checked against the common
feasible set $\mathcal H_\eta$.
Performance is measured by
$U^{\rm num}(M)-U^{\rm num}(M-h)$,
separately from the midpoint scores used by FD
and from the verified decrease margins.

\begin{table}[htbp]
	\centering
	\caption{Observed upper-estimate reductions at $\eta=0.04$,
		in units of $10^{-3}$.}
	\label{tab:five-matrices-four-methods}
	\small
	\begin{tabular}{cccccc}
		\toprule
		Direction & $\rho$ & $z$-LP & $\gamma$-LP & Eigenvalue & FD\\
		\midrule
		1 & 0.20 & 0.627796 & 0.627796 & 1.017332 & 1.026034\\
		2 & 0.20 & 0.816841 & 0.717112 & 1.218438 & 2.526045\\
		3 & 0.35 & 0.757443 & 0.695128 & 0.978708 & 1.729301\\
		4 & 0.20 & 0.927790 & 0.723088 & 1.092553 & 2.346158\\
		5 & 0.40 & 0.771523 & 0.719848 & 1.495004 & 1.686692\\
		\bottomrule
	\end{tabular}
\end{table}

Across all $30$ matrix--budget cases, FD attains the largest
observed reduction, with ties with eigenvalue ranking
in three cases at $\eta=0.002$.
The $z$-LP improves on the $\gamma$-LP in Directions~2--5;
their outcomes in Direction~1 are numerically indistinguishable.
The $z$-LP exceeds eigenvalue ranking only in Direction~4
at $\eta=0.1$ and $0.2$.
All returned $z$-LP interventions have $\gamma_v(h)=0$.

For Direction~4, the Spearman correlation between eigenvalue
and FD rankings decreases from $0.999979$ at $\rho=0.10$
to $0.656553$ at $\rho=0.20$.
At the latter amplitude their top-five sets are disjoint.
Full budget sweeps and ranking details are included in the
reproducibility data.

% ----------------------------------------------------------------
\subsection{Verification with retained and refined profiles}
\label{subsec:supp-numerical-verification-cost}
% ----------------------------------------------------------------

All $120$ archived interventions are retained and checked using
the same left-profile refinement procedure.
The baseline lower bound is
\begin{equation}
	\label{eq:supp-boundary-corrected-decrease}
	L_0=\ell_M(u),\qquad
	m_{\rm fixed}(h)=L_0-q_{M-h}(v).
\end{equation}
For each intervention, a positive left profile $v_h$ is recovered
by LP bisection with tolerance $10^{-7}$ and positivity floor
$10^{-10}$. The initial interval is given by the retained right
and left trial values of the rounded operator
$\widetilde M_h=\operatorname{fl}(M-h)$.
All methods transfer the baseline upper LP basis as a warm start.
Profile recovery maximizes the smallest component at an upper
estimate relaxed by the first accepted value in
$\{10^{-7},3\times10^{-7},10^{-6},3\times10^{-6},10^{-5}\}$.
Acceptance requires strict positivity and an evaluated upper
trial bound at most the numerical estimate plus the chosen
relaxation and $10^{-7}$.
The returned component ratios are checked independently of
solver status. This construction changes neither the archived
interventions nor their observed numerical reductions.

The final certificate is evaluated for the exact rational
difference of the stored binary64 inputs $M-h$, rather than
for $\widetilde M_h$. Define
\begin{equation}
	\label{eq:refined-control-margin}
	\begin{aligned}
		U_{\rm best}(h)
		&=\min\{q_{M-h}(v),q_{M-h}(v_h)\},\\
		m_{\rm best}(h)&=L_0-U_{\rm best}(h).
	\end{aligned}
\end{equation}
Exact rational evaluation selects the smaller upper bound.
Equation~\eqref{eq:framework-certified-decrease} gives
\[
\lambdaup_C(M,I)-\lambdaup_C(M-h,I)\ge m_{\rm best}(h).
\]

Independent exact rational verification confirms admissibility
of all $120$ interventions. The retained baseline profile gives
positive margins for the $60$ LP interventions and none of the
$60$ ranking interventions. After common refinement, all $120$
interventions have positive margins, including all $30$ eigenvalue
and all $30$ FD interventions. The refined profile itself yields
a positive margin in every case; comparison of the two upper
bounds selects the refined profile in $108$ cases and the
baseline profile in $12$.

\begin{table}[htbp]
	\centering
	\small
	\caption{Verified decrease margins after common profile refinement
		at $\eta=0.04$, in units of $10^{-3}$.
		Values are rounded displays of the exact rational margins
		$m_{\rm best}(h)$.}
	\label{tab:refined-control-margins}
	\begin{tabular}{lrrrr}
		\toprule
		Operator & $z$-LP & $\gamma$-LP & Eigenvalue & FD\\
		\midrule
		$M_1$ & 0.627636 & 0.627636 & 1.017019 & 1.025693 \\
		$M_2$ & 0.816422 & 0.716728 & 1.218014 & 2.525667 \\
		$M_3$ & 0.757108 & 0.694880 & 0.978454 & 1.729029 \\
		$M_4$ & 0.927394 & 0.722710 & 1.092216 & 2.345754 \\
		$M_5$ & 0.771161 & 0.719487 & 1.494578 & 1.686322 \\
		\bottomrule
	\end{tabular}
\end{table}

\begin{table}[htbp]
	\centering
	\small
	\caption{Direction~4 at $\eta=0.04$: fixed-profile and refined
		verification. Margins are in units of $10^{-3}$; times are in
		milliseconds. Refinement times are medians of three repetitions;
		exact-check times are measured once and include feasibility and
		both trial upper bounds. Baseline preparation is excluded.}
	\label{tab:direction4-refined-verification}
	\begin{tabular}{lrrrr}
		\toprule
		Method & $10^3m_{\rm fixed}$ & $10^3m_{\rm best}$
		& Refinement & Exact check\\
		\midrule
		$z$-LP & 0.926651 & 0.927394 & 124.67 & 46.27 \\
		$\gamma$-LP & 0.721932 & 0.722710 & 120.84 & 45.97 \\
		Eigenvalue & -0.001100 & 1.092216 & 57.49 & 45.42 \\
		FD & -0.001100 & 2.345754 & 133.44 & 52.45 \\
		\bottomrule
	\end{tabular}
\end{table}

\paragraph{Computational cost.}
Profile refinement and exact verification are reported separately
from intervention construction. Refinement times include model
construction, bisection, positive-profile recovery, and ratio checks.
Each repetition starts from the same baseline basis; no basis is
transferred between different interventions or repetitions.
The baseline basis preparation is recorded separately.
Full results, including all budgets and repetition times, are
provided in the reproducibility data.
Archived construction timings and new refinement timings are
reported separately because they were measured in different
software environments. The environments and raw timings are
included in the reproducibility data.

% ================================================================
\section*{Funding}
% ================================================================

No external funding was received for this work.

% ================================================================
\section*{Data availability}
% ================================================================

The Cora citation-network data are publicly available from the
source cited in the manuscript. Synthetic data generation and
experimental protocols are described in the paper and supplementary
material. Code, stored numerical outputs, exact rational verification
scripts, learning protocols, and scripts generating the final figures
are available in the Zenodo reproducibility package
\cite{IlyasovValeevZenodo2026}, Version~2.0.0,
\href{https://doi.org/10.5281/zenodo.23150973}
{doi:10.5281/zenodo.23150973}.
The package README specifies reproduction procedures,
the scope of verification, and software-environment information.

% ================================================================
\section*{Declaration of competing interest}
% ================================================================

The authors declare no competing financial or personal interests
that could have influenced the work reported in this article.

% ================================================================
\section*{Declaration of generative AI and AI-assisted technologies
	in the manuscript preparation process}
% ================================================================

The authors used ChatGPT (OpenAI) to assist with English-language
editing, manuscript organization, and computational implementation.
The authors reviewed and revised the assisted material and take
full responsibility for the article, including its mathematical
statements, proofs, code, numerical results, and references.

% References follow immediately.


\begin{thebibliography}{99}
	
	% ================================================================
	% Books and classical spectral theory
	% ================================================================
	
	\bibitem{Berman}
	A.~Berman and R.~J.~Plemmons,
	\emph{Nonnegative Matrices in the Mathematical Sciences},
	Classics in Applied Mathematics, vol.~9,
	Society for Industrial and Applied Mathematics,
	Philadelphia, PA, 1994.
	
	\bibitem{BirkhoffVarga1958}
	G.~Birkhoff and R.~S.~Varga,
	Reactor criticality and nonnegative matrices,
	\emph{J. Soc. Indust. Appl. Math.}
	\textbf{6} (1958), 354--377.
	
	\bibitem{DongZhangWangRQGNN}
	X.~Dong, X.~Zhang and S.~Wang,
	Rayleigh quotient graph neural networks for graph-level anomaly detection,
	in \emph{Proceedings of the Twelfth International Conference on Learning
		Representations (ICLR)}, 2024.
	
	\bibitem{Friedland1990}
	S.~Friedland,
	Characterizations of the spectral radius of positive operators,
	\emph{Linear Algebra Appl.}
	\textbf{134} (1990), 93--105.
	
	\bibitem{BauerFike1960}
	F.~L.~Bauer and C.~T.~Fike,
	Norms and exclusion theorems,
	\textit{Numer. Math.} \textbf{2} (1960), 137--141.
	\href{https://doi.org/10.1007/BF01386217}
	{doi:10.1007/BF01386217}.

	\bibitem{GreenbaumLiOverton2020}
	A.~Greenbaum, R.-C.~Li and M.~L.~Overton,
	First-order perturbation theory for eigenvalues and eigenvectors,
	\textit{SIAM Rev.} \textbf{62} (2020), 463--482.
	\href{https://doi.org/10.1137/19M124784X}
	{doi:10.1137/19M124784X}.

	\bibitem{Elsner1985}
	L.~Elsner,
	An optimal bound for the spectral variation of two matrices,
	\textit{Linear Algebra Appl.} \textbf{71} (1985), 77--80.
	\href{https://doi.org/10.1016/0024-3795(85)90236-8}
	{doi:10.1016/0024-3795(85)90236-8}.
	
	\bibitem{VanMieghemEtAl2011}
	P.~Van Mieghem, D.~Stevanovi\'c, F.~Kuipers, C.~Li,
	R.~van de Bovenkamp, D.~Liu and H.~Wang,
	Decreasing the spectral radius of a graph by link removals,
	\textit{Phys. Rev. E} \textbf{84} (2011), 016101.
	\href{https://doi.org/10.1103/PhysRevE.84.016101}
	{doi:10.1103/PhysRevE.84.016101}.
	
	\bibitem{ChenEtAl2016}
	C.~Chen, H.~Tong, B.~A.~Prakash, T.~Eliassi-Rad,
	M.~Faloutsos and C.~Faloutsos,
	Eigen-optimization on large graphs by edge manipulation,
	\textit{ACM Trans. Knowl. Discov. Data}
	\textbf{10} (2016), no.~4, Article~49, 1--30.
	\href{https://doi.org/10.1145/2903148}
	{doi:10.1145/2903148}.
	% ================================================================
	% Extended Rayleigh and minimax methods
	% ================================================================
	
	
	
	\bibitem{IlyasovValeevZenodo2026}
	[dataset] Y.~S.~Il'yasov and N.~F.~Valeev,
	Reproducibility package for
	``Cone Rayleigh Levels: Finite Perturbations and Certified Control'',
	Version~2.0.0, Zenodo, 2026.
	\href{https://doi.org/10.5281/zenodo.23150973}
	{doi:10.5281/zenodo.23150973}.
	
	\bibitem{Ilyasov2001}
	Y.~S.~Il'yasov,
	On positive solutions of indefinite elliptic equations,
	\emph{C. R. Acad. Sci. Paris S\'er. I Math.}
	\textbf{333} (2001), no.~6, 533--538.
	
	\bibitem{IlyasFunc2007}
	Y.~S.~Il'yasov,
	Bifurcation calculus by the extended functional method,
	\emph{Funct. Anal. Appl.}
	\textbf{41} (2007), no.~1, 18--30.
	
	\bibitem{IlIvan1}
	Y.~S.~Il'yasov and A.~A.~Ivanov,
	Computation of maximal turning points to nonlinear equations by nonsmooth
	optimization,
	\emph{Optim. Methods Softw.}
	\textbf{31} (2016), 1--23.
	
	\bibitem{ilyasov2021}
	Y.~S.~Il'yasov,
	Finding saddle-node bifurcations via a nonlinear generalized
	Collatz--Wielandt formula,
	\emph{Internat. J. Bifur. Chaos Appl. Sci. Engrg.}
	\textbf{31} (2021), no.~1, 2150008.
	
	\bibitem{IlyasJDE24}
	Y.~S.~Il'yasov,
	A finding of the maximal saddle-node bifurcation for systems of differential
	equations,
	\emph{J. Differential Equations}
	\textbf{378} (2024), 610--625.
	
	\bibitem{Ilyasov2026}
	Y.~S.~Il'yasov,
	On the minimax bifurcation formula,
	arXiv preprint arXiv:2605.17331, 2026.
	
	\bibitem{IlyasovValeev2024}
	Y.~S.~Il'yasov and N.~F.~Valeev,
	An extension of the Perron--Frobenius theory to arbitrary matrices and cones,
	\emph{Electron. J. Linear Algebra}
	\textbf{40} (2024), 788--802.
	
	\bibitem{IlyasovValeevPencils2026}
	Y.~S.~Il'yasov and N.~F.~Valeev,
	Cone minimax principles for non-selfadjoint operator pencils,
	arXiv preprint arXiv:2606.31129, 2026.
	
	\bibitem{IvanIlya}
	A.~A.~Ivanov and Y.~S.~Il'yasov,
	Finding bifurcations for solutions of nonlinear equations by quadratic
	programming methods,
	\emph{Comput. Math. Math. Phys.}
	\textbf{53} (2013), 350--364.
	
	
	% ================================================================
	% Graph learning and directed networks
	% ================================================================
	
	
	
	\bibitem{MaEtAl2019}
	Y.~Ma, J.~Hao, Y.~Yang, H.~Li, J.~Jin and G.~Chen,
	Spectral-based graph convolutional network for directed graphs,
	arXiv preprint arXiv:1907.08990, 2019.
	
	\bibitem{PerlmutterEtAl2023}
	M.~Perlmutter, A.~Tong, F.~Gao, G.~Wolf and M.~Hirn,
	Understanding graph neural networks with generalized geometric scattering
	transforms,
	\emph{SIAM J. Math. Data Sci.}
	\textbf{5} (2023), no.~4, 873--898.
	
	\bibitem{RossiEtAl2024}
	E.~Rossi, B.~Charpentier, F.~Di Giovanni, F.~Frasca,
	S.~G\"unnemann and M.~M.~Bronstein,
	Edge directionality improves learning on heterophilic graphs,
	in \emph{Proceedings of the Second Learning on Graphs Conference},
	\emph{Proc. Mach. Learn. Res.}
	\textbf{231} (2024), 25:1--25:27.
	
	\bibitem{SenEtAl2008}
	P.~Sen, G.~Namata, M.~Bilgic, L.~Getoor, B.~Gallagher and T.~Eliassi-Rad,
	Collective classification in network data,
	\emph{AI Mag.}
	\textbf{29} (2008), no.~3, 93--106.
	
	
	
	\bibitem{TongEtAl2020}
	Z.~Tong, Y.~Liang, C.~Sun, D.~S.~Rosenblum and A.~Lim,
	Directed graph convolutional network,
	arXiv preprint arXiv:2004.13970, 2020.
	
	
	
	\bibitem{ZhangMagNet2021}
	X.~Zhang, Y.~He, N.~Brugnone, M.~Perlmutter and M.~Hirn,
	MagNet: A neural network for directed graphs,
	in \emph{Advances in Neural Information Processing Systems},
	\textbf{34} (2021), 27003--27015.
	
	
	% ================================================================
	% Applications and minimax theory
	% ================================================================
	
	\bibitem{SalazarIlaysov2020}
	P.~D.~P.~Salazar, Y.~S.~Il'yasov, L.~F.~C.~Alberto,
	E.~C.~M.~Costa and M.~B.~Salles,
	Saddle-node bifurcations of power systems in the context of variational
	theory and nonsmooth optimization,
	\emph{IEEE Access}
	\textbf{8} (2020), 110986--110993.
	
	\bibitem{Sion1958}
	M.~Sion,
	On general minimax theorems,
	\emph{Pacific J. Math.}
	\textbf{8} (1958), 171--176.
	
\end{thebibliography}
\end{document}